\RequirePackage{fix-cm}
\documentclass[
  smallextended,
  envcountsect,
]{svjour3}
\smartqed  
\usepackage{graphicx}
\usepackage{epstopdf}
\usepackage{amsmath}
\usepackage{amssymb}
\usepackage{xcolor}
\usepackage{centernot}
\usepackage{nicefrac}
\usepackage{bbm}
\usepackage{thmtools}
\usepackage{thm-restate}
\usepackage{hyperref}
\usepackage{cleveref}

\newcommand*\diff{\mathop{}\!\mathrm{d}}
\renewcommand{\vec}[1]{\mathbf{#1}}

\newcommand{\PP}{\mathcal{P}}
\newcommand{\E}{\mathrm{E}}
\renewcommand{\H}{\mathcal{H}}
\newcommand{\x}{\vec x}
\newcommand{\y}{\vec y}

\newcommand{\norm}[1]{\left\lVert#1\right\rVert}
\newcommand{\wrt}{w.\,r.\,t.}

\newcommand{\phim}{\boldsymbol \phi_m}
\newcommand{\M}{\mathcal{M}([0,1]^N)}

\newtheorem{defi}{Definition}

\usepackage[verbose,lmargin=2cm,rmargin=2cm,marginpar=0pt,marginparsep=0pt]{geometry}
\begin{document}
	
	\title{
		On generalization in moment-based domain adaptation}
	
	
	\author{Werner Zellinger         \and
		Bernhard A. Moser \and
		Susanne Saminger-Platz
	}
	
	
	\institute{Werner Zellinger and Bernhard A. Moser \at
		Data Science\\
		Software Competence Center Hagenberg GmbH, Austria\\
		\email{\{werner.zellinger, bernhard.moser\}@scch.at}
		\and
	    Susanne Saminger-Platz \at
		Department of Knowledge-Based Mathematical Systems\\
		Johannes Kepler University Linz, Austria\\
		\email{susanne.saminger-platz@jku.at}
	}
	
	\date{Received: date / Accepted: date}
	
	\maketitle
	
	\begin{abstract}
		Domain adaptation algorithms are designed to minimize the misclassification risk of a discriminative model for a target domain with little training data by adapting a model from a source domain with a large amount of training data.
		Standard approaches measure the adaptation discrepancy based on distance measures between the empirical probability distributions in the source and target domain.
		In this setting, we address the problem of deriving generalization bounds under practice-oriented general conditions on the underlying probability distributions.
		As a result, we obtain generalization bounds for domain adaptation based on finitely many moments and smoothness conditions.
	\end{abstract}
	
	\section{Motivation}
	\label{sec:intro}
	
	Domain adaptation problems are encountered in everyday life of engineering machine learning applications whenever there is a discrepancy between assumptions on the learning and application setting.
	For example, most theoretical and practical results in statistical learning are based on the assumption that the training and test sample are drawn from the same distribution.
	As outlined in~\cite{blitzer2008learning,ben2010theory,pan2010survey,ben2014domain}, however, this assumption may be violated in typical applications such as natural language processing~\cite{blitzer2007biographies,jiang2007instance} and computer vision~\cite{ganin2016domain,zellinger2017robust,zellinger2016linear}. 
	
	In this work, we relax the classical assumption of identical distributions under training and application setting by postulating that only a finite number of moments of these distributions are aligned.
	
	This postulate is motivated two-fold:
	First, by the methodology to overcome a present difference in distributions by mapping the samples into a latent model space where the resulting corresponding distributions are aligned.
	See Figure~\ref{fig:grafical_abstract} for an illustration.
	Moment-based algorithms perform particularly well in many practical tasks~\cite{duan2012domain,baktashmotlagh2013unsupervised,sun2016deep,zellinger2017central,zellinger2017robust,koniusz2017domain,li2018adaptive,zhao2017joint,nikzad2018domain,peng2018cross,ke2018identity,Wei2018GenerativeAG,xing2018adaptive,peng2018moment,peng2019weighted}.
	Second, by the current scientific discussion about the choice of an appropriate distance function for domain adaptation~\cite{ben2007analysis,courty2017optimal,long2015learning,long2016unsupervised,Zhuang2015supervised,ganin2016domain}.
	The convergence in most common probability metrics of compactly supported distributions implies the convergence of finitely many moments.
	In particular, many common probability metrics admit upper bounds on moment-based distances, see Figure~\ref{fig:metrics}.
	Therefore, results under the proposed setting can also give theoretical insights to approaches based on stronger concepts of similarity like the Wasserstein distance~\cite{courty2017optimal}, the Maximum Mean Discrepancy~\cite{long2016unsupervised} or the f-divergences~\cite{Zhuang2015supervised}.
	
	However, distributions with only finitely many similar moments can be very different, see e.g.~\cite{lindsay2000moments}, which implies that classical bounds on the target risk are very loose for general distributions under the proposed setting.
	This brings us to our motivating question \textit{under which further conditions can we expect a discriminative model to perform well on a future test sample given that only finitely many moments are aligned with those of a prior training sample.}
	
	\begin{figure}
	    \begin{center}
	        \includegraphics[width=0.5\linewidth]{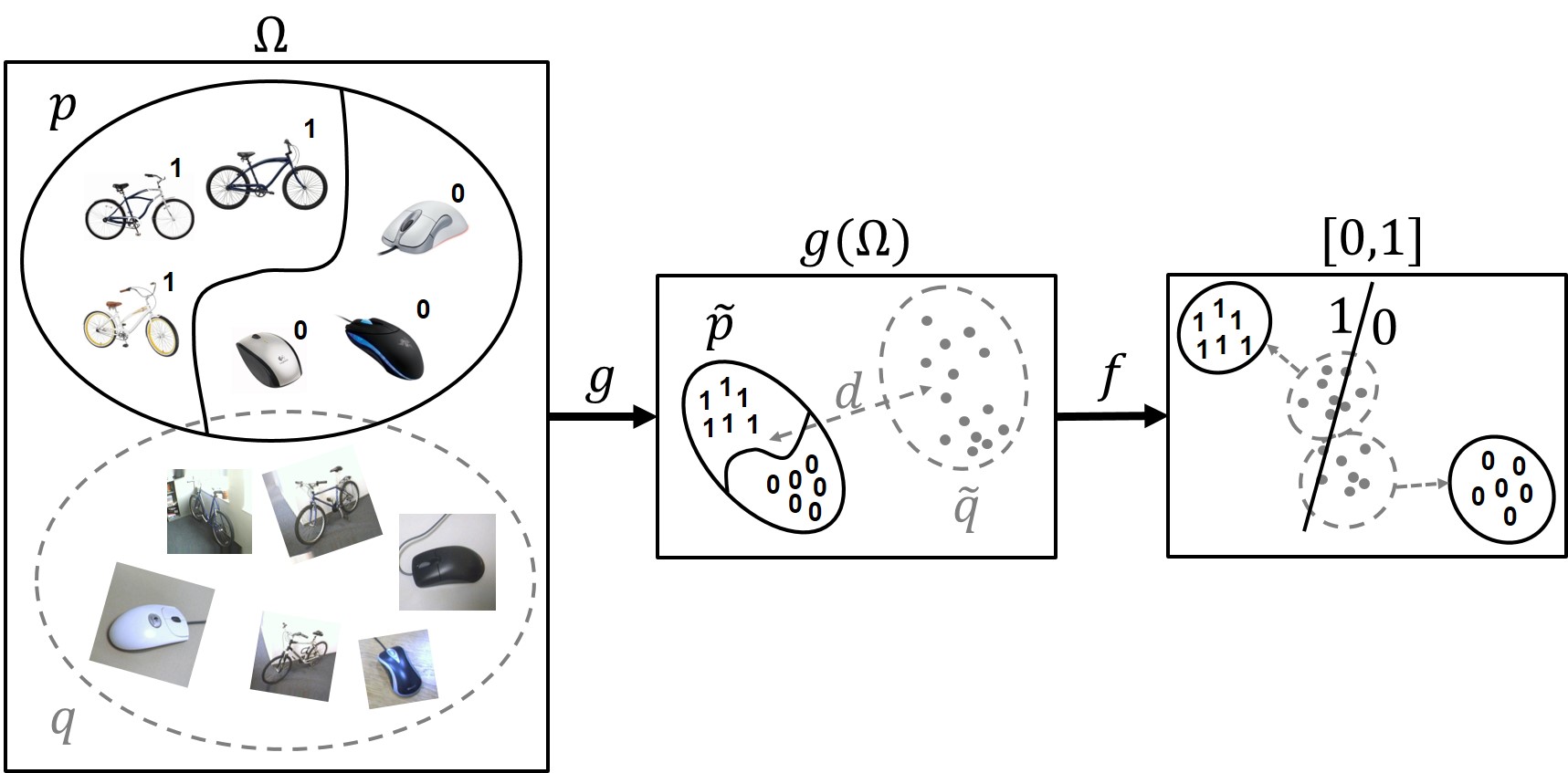}
	    \end{center}
		\caption{Illustration of machine learning model $f\circ g:\Omega\to \{0,1\}$ for unsupervised domain adaptation. Given: {\it Unlabeled} sample following target probability density $q$ and labeled sample following auxiliary probability density $p$; Goal: High performance on target density $q$; Method: Minimizing error on $p$ and moment-distance $d$ between the samples densities $\tilde p$ and $\tilde q$ in the latent space $g(\Omega)$.
		}
		\label{fig:grafical_abstract}
	\end{figure}
	
	We approach this problem by also considering the information encoded in the distributions in addition to the moments.
	Following Information Theory, this information can be modeled by the deviation of the differential entropy to the maximum entropy distribution~\cite{cover2012elements,milev2012moment}, or equivalently, by the error in Kullback-Leibler divergence (KL-divergence) of approximation by exponential families~\cite{csiszar1975}.
	Note that exponential families are the only parametric distributions with fixed compact support having the property that a finite pre-defined vector of moments can serve as sufficient statistic~\cite{koopman1936distributions} and therefore carries all the information about the distribution.
	In addition, exponential families are particularly suitable for our analysis as they include Truncated Normal Distributions arising in many applications.
	
	\begin{figure}[ht]
	    \centering
		\includegraphics[width=0.75\linewidth]{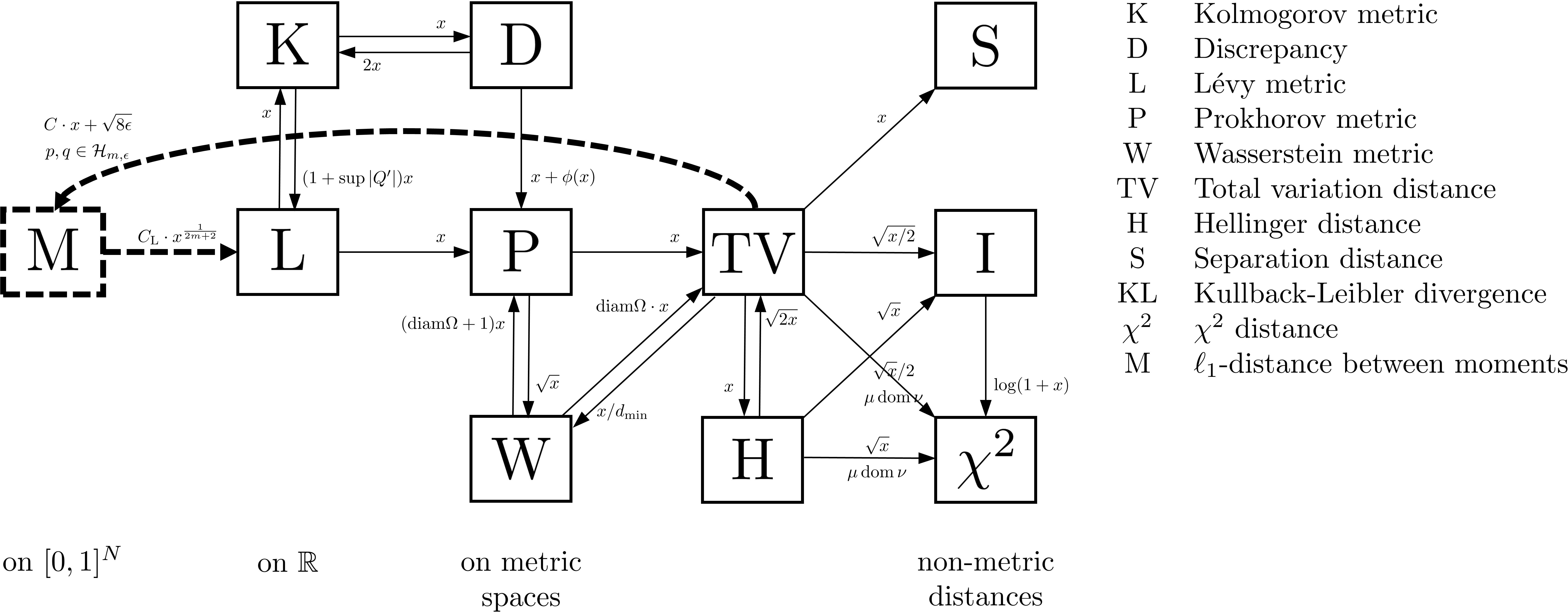}
		\caption{Relationships among probability metrics arranged from weaker (left) to stronger (right) as illustrated in~\cite{gibbs2002choosing} and supplemented by Theorem~\ref{thm:bound_for_smooth_functions} and Lemma~\ref{lemma:rachev_bound} (dashed).
		A directed arrow from $\text{A}$ to $\text{B}$ annotated by a function $h(x)$ means that $d_\text{A}\leq h(d_\text{B})$.
		For other notations, restrictions and applicability see~\cite{gibbs2002choosing}.
		}
		\label{fig:metrics}
	\end{figure}
	
	We analyze the convergence of sequences of smooth probability densities in terms of finite moment convergence and the differential entropy of the densities.
	Based on results about the approximation by maximum entropy distributions and polynomials~\cite{barron1991approximation,cox1988approximation} we provide (locally admissible) bounds of the form
	\begin{align}
	\label{eq:form_of_bounds}
	\norm{p-q}_{L^1}\leq C\cdot\norm{\boldsymbol{\mu}_p-\boldsymbol{\mu}_q}_1+\varepsilon,
	\end{align}
	where $\norm{p-q}_{L^1}$ is the $L^1$-difference between the probability densities $p$ and $q$ with respective pre-defined vectors of (sample) moments $\boldsymbol{\mu}_p$ and $\boldsymbol{\mu}_q$,
	$C$ is a constant depending on the smoothness of $p$ and $q$ and
	$\epsilon$ is the error of approximating $p$ and $q$ by (estimators of) maximum entropy distributions measured in terms of differential entropy (and sample size).
	The term $\epsilon^2/2$ can be interpreted as upper bound on the amount of information lost when representing $p$ (or $q$) by its moments $\boldsymbol{\mu}_p$ (or $\boldsymbol{\mu}_q$).
	
	To obtain bounds on the expected misclassification risk of a discriminative model tested on a sample with only finitely many moments similar to those of the training sample, we extend the theoretical bounds proposed in~\cite{ben2010theory} by means of Eq.~\eqref{eq:form_of_bounds}.
	The resulting learning bounds do not make assumptions on the structure of the underlying (unknown) labeling functions.
	In the case of two underlying labeling functions, we obtain error bounds that are relative to the performance of some optimal discriminative function and in the case of one underlying labeling function, i.e. in the covariate-shift setting~\cite{sugiyama2012machine,ben2014domain}, we obtain absolute error bounds.
	
	Our bounds show that a small misclassification risk of the discriminative model can be expected if the misclassification risk of the model on the training sample is small, if the samples are large enough and their densities have high entropy in the respective classes of densities sharing the same finite collection of moments.
	
	As an application, we give bounds on the misclassification risk of some recently proposed moment-based algorithms for unsupervised domain adaptation~\cite{zellinger2017central,zellinger2017robust,nikzad2018domain} illustrated in Figure~\ref{fig:grafical_abstract}.
	Our bounds are uniform for a class of smooth distributions and multivariate moments with solely univariate terms.
	
	This work is structured as follows: Section~\ref{sec:related} describes some related works on domain adaptation, moment-based bounds on distances between distributions and exponential families, Section~\ref{sec:preliminaries} gives the basic notations and preliminaries used to prove our results, Section~\ref{sec:problem} formulates the problem considered in this work, Section~\ref{sec:results} discusses our approach based on convergence rate analysis, Section~\ref{sec:main_result_on_learning_bounds} gives our main result on moment-based learning bounds, Section~\ref{sec:application} applies our result on moment-based algorithms for unsupervised domain adaptation and Section~\ref{sec:details} gives all proofs.
	
	\section{Related Work}
	\label{sec:related}
	
	Most error bounds for classes of discriminative models in statistical learning theory~\cite{vapnik2013nature} are based on the assumption that the training and test sample are drawn from the same distribution and that an underlying labeling function exists.
	
	Ben-David et al.~\cite{ben2007analysis,blitzer2008learning,ben2010theory,ben2010impossibility} extended this theory to a basic formal model of domain adaptation.
	The definition of the domain adaptation problem assumes a training sample with a distribution different from that of a test sample and the existence of two corresponding labeling functions.
	They propose bounds on the misclassification probability of discriminative models for domain adaptation.
	Their bounds are based on the model's misclassification probability on the training sample, a distance between the training and the test sample and the misclassification risk of a reference model that performs well on both distributions.
	Their work includes a bound based on the $L^1$-norm of the difference between the samples densities.
	In~\cite{ben2010impossibility} they show that a high dissimilarity of the distributions makes effective domain adaptation impossible in general situations.
	
	Mansour et al.~\cite{mansour2009domain,mansour2009multiple,mansour2014robust} extended the arguments of Ben-David et al. by more general distance measures~\cite{mansour2009domain}, robustness concepts of algorithms~\cite{mansour2014robust} and tighter error bounds based on the Rademacher complexity.
	
	Recently, Vural~\cite{vural2018bound} considered the problem of transforming two differently distributed samples by means of two different functions in a common latent space and subsequently learn a discriminative model.
	Her assumptions imply that the two different functions do not map differently labeled sample points onto the same point in the latent space.
	
	One assumption commonly made in domain adaptation is the {\it covariate shift} assumption~\cite{sugiyama2005generalization,sugiyama2012machine,ben2014domain} stating one underlying labeling function.
	This assumption is partially motivated by the impossibility of overcoming an error induced by a difference of two general labeling functions, corresponding to the two distributions, in unsupervised domain adaptation~\cite{ben2010theory}.
	
	Following the works mentioned above, questions about the difference of two distributions based on finitely many moments arise.
	The literature about Moment Problems~\cite{akhiezer1965classical,tardella2001note,kleiber2013multivariate,schmudgen2017moment} provides bounds on the difference between two one-dimensional distributions with finitely many coinciding moments.
	However, bounds in the multivariate case remain scarce~\cite{laurent2009sums,di2018multidimensional}.
	
	Lindsay and Basak show~\cite{lindsay2000moments} that the difference between two distributions with finitely many coinciding moments can be very large.
	
	Tagliani et al.~\cite{tagliani2003note,tagliani2002entropy,tagliani2001numerical,milev2012moment} show that, in the case of compactly supported distributions, this difference can be bounded by means of the KL-divergence between the distributions and maximum entropy distributions sharing the same finite collection of moments.
	
	Barron and Sheu~\cite{barron1991approximation} give bounds on the KL-divergence between a compactly supported probability density and its approximation by estimators of maximum entropy distribution.
	They establish rates of convergence for log-density functions assumed to have square integrable derivatives.
	Their analysis involves moment-based bounds.
	\\
	
	Our work is partly motivated by the high performance of moment-based unsupervised domain adaptation methods.
	Recent examples can be found in the areas of deep learning~\cite{zellinger2017central,sun2016deep,koniusz2017domain,li2018adaptive,peng2018cross,ke2018identity,Wei2018GenerativeAG,xing2018adaptive}, kernel methods~\cite{duan2012domain,baktashmotlagh2013unsupervised} and linear regression~\cite{nikzad2018domain}.
	However, none of these works provide theoretical guarantees for a small misclassification risk with exception of~\cite{peng2018moment,zellinger2017robust} in which loose bounds (as a consequence of considering general distributions) are proposed.
	Another motivation of our work is that many common probability metrics admit upper bounds on moment-based distance measures~\cite{rachev2013methods}.
	Gibbs and Su~\cite{gibbs2002choosing} review different useful relations between probability metrics without considering moment distances.
	
	Our work is based on the observation that bounds on the $L^1$-norm of the difference between densities lead to bounds on the misclassification probability of a discriminative model according to Ben-David et al.~\cite{ben2010theory}.
	Following ideas from Tagliani et al.~\cite{tagliani2003note,tagliani2002entropy,tagliani2001numerical} and properties of maximum entropy distributions~\cite{cover2012elements}, we obtain such bounds for multivariate distributions based on the differential entropy.
	Following Barron and Sheu~\cite{barron1991approximation} and Cox~\cite{cox1988approximation}, appropriate regularity assumptions on the distributions are presented under which the KL-divergence based bounds are further upper bounded in terms of (sample) moment differences leading to the form of Eq.~\eqref{eq:form_of_bounds}.
	
	Our results supplement the picture of probability metrics proposed by Gibbs and Su~\cite{gibbs2002choosing} by moment distances, see Figure~\ref{fig:metrics}.
	In contrast to other works, our main result is a learning bound for domain adaptation that does not depend on the knowledge of a full test sample but only on the knowledge of finitely many of its (sample) moments.
	
	\section{Notations and Preliminaries}
	\label{sec:preliminaries}
	
	Throughout the paper, we assume that all distributions are represented by probability density functions \wrt~the Lebesgue reference measure.
	We denote by $\mathcal{M}(\Omega)$ the set of all probability densities \wrt~the Lebesgue reference measure with support $\Omega\subseteq\mathbb{R}^M$.
	A multiset with elements in $\Omega$ is called a {\it $k$-sized sample drawn from $p$}, denoted by $X_p$, if its elements are realization of $k$ independently identically distributed random variables with probability density function $p$.
	We denote by $\mathbb{R}_m[x_1,\ldots,x_N]$ the set of polynomials with degree up to $m$ in $N$ variables $x_1,\ldots,x_N$.
	Column vectors are denoted by bold symbols, e.g. $\vec x=(x_1,\ldots,x_N)^\text{T}$.
	
	\subsection{Statistical Learning Theory}
	\label{subsec:learning_theory}
	
	Following~\cite{vapnik2013nature}, we formulate the problem of binary classification on an input set $\Omega\subseteq\mathbb{R}^M$:
	Consider a probability density $q\in\mathcal{M}(\Omega)$ and a labeling function $l:\Omega\to [0,1]$, which can have intermediate (expected) values if labeling occurs non-deterministically.
	Given a $k$-sized sample $X_q$ drawn from $q$, the goal of binary classification is to find a discriminative model $f$ from a function class
	\begin{align}
	\label{eq:discriminative_model}
	\mathcal{F} &\subseteq \left\{f:\Omega\to \{0,1\}\mid f~ \text{integrable}\right\}
	\end{align}
	with a small misclassification risk
	\begin{equation}
	\label{eq:misclassification_error}
	\E_{q}\big[|f-l|\big]=\int_{\Omega} |f(\x)-l(\x)|\, q(\x) d\x.
	\end{equation}
	The Vapnik-Chervonenkis dimension (VC-dimension) $d$ of a function class $\mathcal{F}$ defined in Eq.~\eqref{eq:discriminative_model} is the maximum cardinality $|X|$ of a set of non-collinear points $X\subseteq \Omega$ such that for all labeling functions $l:\Omega\to \{0,1\}$ there exists a model $f\in\mathcal{F}$ with zero misclassification risk on the set $X$, i.e. $\sum_{x\in X} |f(x)-l(x)|=0$~\cite{vapnik2013nature}.
	
	\sloppy
	According to Vapnik and Chervonenkis~\cite{vapnik2015uniform}, the following holds with probability at least $1-\delta$ (over the choice of $k$-sized samples $X_q$ drawn from $q$):
	\begin{align}
	\label{eq:vapnik_bound}
	\sup_{f\in\mathcal{F}}\Big|\E_{q}\big[|f-l|\big]- \frac{1}{k}\sum_{\x\in X_q}|f(\x)-l(\x)|\Big|\leq \sqrt{\frac{4}{k} \left( d\log \frac{2 e k}{d} + \log\frac{4}{\delta} \right)}
	\end{align}
	The left-hand side of Eq.~\eqref{eq:vapnik_bound} is called the {\it generalization error} of $\mathcal{F}$.
	According to Eq.~\eqref{eq:vapnik_bound}, a model $f\in\mathcal{F}$ can be expected to perform well on a large enough sample $X_q$ if the empirical misclassification risk ${\frac{1}{k}\sum_{\x\in X_q}|f(\x)-l(\x)|}$ is small.
	However, in domain adaptation, samples from two different distributions are considered~\cite{ben2010theory,vapnik2013nature}.
	
	\subsection{Domain Adaptation}
	\label{subsec:domain_adaptation}
	
	In domain adaptation~\cite{daume2006domain,ben2010theory,sugiyama2012machine}, we consider two different distributions represented by probability densities $p,q\in\mathcal{M}(\Omega)$.
	Following~\cite{ben2010theory}, we consider two corresponding unknown integrable {\it labeling} functions $l_p,l_q:\Omega\to [0,1]$.
	Given two $k$-sized samples $X_p$ and $X_q$ drawn from $p$ and $q$, respectively, and some subsets $Y_p\subseteq l_p(X_p), Y_q\subseteq l_q(X_q)$ of the labels, the goal of domain adaptation is to find an $f\in\mathcal{F}$ with a low misclassification risk as defined in Eq.~\eqref{eq:misclassification_error} and $\mathcal{F}$ as defined in Subsection~\ref{subsec:learning_theory}.
	As Ben-David et al. showed in~\cite{ben2010theory}, the following holds:
	\begin{align}
	\label{eq:simple_da_equation}
	\E_{q}\big[|f-l_q|\big]\leq \E_{p}\big[|f-l_p|\big] + \norm{p-q}_{L^1} + \lambda^*
	\end{align}
	where $\lambda^* = \inf_{h\in\mathcal{F}}\big(\E_p[|h-l_p|]+\E_q[|h-l_q|]\big)$.
	The {\it covariate shift} emphasis~\cite{sugiyama2012machine,ben2014domain} states the equality of the two labeling functions, i.e. $l_p=l_q$.
	In the specification of {\it unsupervised} domain adaptation, the label set $Y_q$ is empty and the misclassification risk of interest, i.e. the error on the left hand side of Eq.~\eqref{eq:simple_da_equation}, cannot be sampled making upper bounds as expressed by Eq.~\eqref{eq:simple_da_equation} particularly interesting.
	
	\subsection{Maximum Entropy Distributions}
	\label{subsec:maxent}
	
	Shannon's differential entropy $h(p)$ of a probability density $p\in\M$ is given by the functional
	\begin{align}
	\label{eq:entropy}
	h(p)=-\int_{[0,1]^N} p(\x) \log p(\x)\, d \x
	\end{align}
	where $\log$ is the natural logarithm~\cite{cover2012elements}.
	The differential entropy is concave, may be negative, and may be potentially infinite if the integral in Eq.~\eqref{eq:entropy} diverges.
	
	For the rest of this work let $\psi(m,N)$ denote the number of monomials of maximum total degree $m$ in $N$ variables excluding the monomial $1$ of degree $0$.
	Note that the number $\zeta(m,N)$ of monomials of total degree $m$ in $N$ variables is equal to the number of weak compositions and therefore $\zeta(m,N)=\binom{N+m-1}{m}$.
	It follows that $\psi(m,N)=\sum_{i=1}^m \zeta(i,N)=\binom{N+m}{m}-1$.
	
	Consider some $\boldsymbol{\phi}=(\phi_1,\ldots,\phi_{\psi(m,N)})^\text{T}$ such that $1,\phi_1,\ldots,\phi_{\psi(m,N)}$ is a basis of $\mathbb{R}_m[x_1,\ldots,x_N]$.
	By the compactness of the support of $p$, the moments
	\begin{align}
	\int\boldsymbol{\phi} p:=\left(\int_{[0,1]^N} \phi_1(\x) p(\x) d\x,\ldots,\int_{[0,1]^N}\phi_{\psi(m,N)}(\x)p(\x)d\x\right)^\text{T}
	\end{align}
	are finite.
	Consider the class
	\begin{align}
	\label{eq:moment_space}
	\PP:=\left\{q\in\M\Bigm\vert \int \boldsymbol{\phi} q=\int\boldsymbol{\phi} p\right\}
	\end{align}
	of densities sharing the same pre-defined moments.	
	The {\it principle of maximum entropy} states that the distribution which best represents the knowledge captured by the moments $\int\boldsymbol{\phi} p$ is that $p^*\in\PP$ having the largest differential entropy~\cite{cover2012elements}.
	This distribution is called {\it maximum entropy distribution} constrained at the moments $\int\boldsymbol{\phi} p$, its probability density is called the {\it maximum entropy density} and will be denoted by $p^*$.
	By the Lebesgue reference measure the density $p$ is not a convex combination of Dirac deltas and, as the elements of $\boldsymbol{\phi}$ form a basis of $\mathbb{R}_m[x_1,\ldots,x_N]$, the maximum entropy density exists~\cite{frontini2011hausdorff,wainwright2008graphical,kleiber2013multivariate}.
	The uniqueness of $p^*$ follows from the concavity of the differential entropy~\cite{cover2012elements,csiszar1975}.
	We denote by $h_{\boldsymbol{\phi}}(p):=h(p^*)$ the entropy of $p^*$.
	It is well known~\cite{csiszar1975} that $p^*=\arg\min_{q\in\mathcal{E}} D(p\Vert q)$ where $D$ refers to the KL-divergence
	\begin{align}
	\label{eq:KL_divergence}
	D(p\Vert q):=\int_{[0,1]^N} p(\x) \log \frac{p(\x)}{q(\x)} d\x
	\end{align}
	and $\mathcal{E}$ is the exponential family consisting of densities of the form
	\begin{align}
	\label{eq:maxent_distr_formula}
	q(\x) = c(\boldsymbol\lambda) \exp\left(-\langle \boldsymbol\lambda,\boldsymbol\phi(\x)\rangle\right)
	\end{align}
	where
	\begin{align}
	    \label{eq:constant_of_normalization}
	    c(\boldsymbol\lambda):=\left(\int_{[0,1]^N} \exp\left(-\langle \boldsymbol\lambda,\boldsymbol\phi(\x)\rangle\right) d\x\right)^{-1}
	\end{align}
	is the constant of normalization, $\boldsymbol\lambda\in\mathbb{R}^{k}$ is a parameter vector and $\langle \x, \y\rangle=x_1 y_1+\ldots+x_k y_k$ is the Euclidean inner product~\cite{cover2012elements,csiszar1975}.
	Consequently, the maximum entropy density $p^*$ can be interpreted as the best approximation of $p$ by densities in $\mathcal{E}$ \wrt~KL-divergence and it is sometimes called {\it information projection} of $p$ onto the space $\mathcal{E}$~\cite{csiszar1975}.
	The KL-divergence (or relative entropy) in Eq.~\eqref{eq:KL_divergence} can be interpreted as the amount of information lost when identifying $p$ with the density $q$~\cite{cover2012elements}.	
	It holds that $D(p\Vert p^*)=h_{\boldsymbol{\phi}}(p)-h(p)$ and that $h_{\boldsymbol{\phi}}(p)\to h(p)$ as $m\to\infty$~\cite{borwein1991convergence,tagliani1999hausdorff}.
	
	\section{Formal Problem Statement} 
	\label{sec:problem}

    We start with a typical scenario encountered in statistical learning theory~\cite{vapnik2013nature} on the one hand and domain adaptation theory~\cite{ben2010theory} on the other hand.
	To this end, we assume source and target densities $p, q\in\mathcal{M}([0,1]^N)$ with corresponding labeling functions $l_p,l_q :[0,1]^N\to[0,1]$ as well as $f\in \mathcal{F}$ from a family of discriminative functions of finite VC-dimension as defined in Subsection~\ref{subsec:domain_adaptation}.
	In this work, furthermore, we postulate the alignment of finitely many moments, i.e. $\int\boldsymbol{\phi} p\approx\int\boldsymbol{\phi} q$ for some $\boldsymbol{\phi}\in\mathbb{R}_m[x_1,\ldots,x_N]^k$.
	
	Our goal is to determine and describe conditions on the densities $p$ and $q$ such that a small target risk $\E_q\big[|f-l_q|\big]$ is induced by a small (sampled) source risk $\E_p\big[|f-l_p|\big]$, a small difference $\norm{\boldsymbol{\mu}_p-\boldsymbol{\mu}_q}_1$ between the (sampled) moments $\boldsymbol{\mu}_p$ and $\boldsymbol{\mu}_p$ and a small distance $\lambda^*$ between the labeling functions $l_p$ and $l_q$ as defined in Eq.~\eqref{eq:simple_da_equation}.
	
	Without further conditions on the densities, a small target risk is not induced by the above mentioned quantities (see Subsection~\ref{subsec:from_moment_convergence_to_l1}).
	Throughout this work, we refer to this problem as the {\it moment adaptation problem} on the unit cube.

	\section{Approach by Convergence Rate Analysis}
	\label{sec:results}
	It will turn out that the postulation of high-entropy distributions satisfying additional smoothness conditions allows us to provide learning bounds.
	Our approach is based on the analysis of the $L^1$-convergence rate of sequences of densities based on the convergence of finitely many of its corresponding moments as motivated in the following.
	
    \subsection{From Moment Similarity to $L^1$-similarity}
	\label{subsec:from_moment_convergence_to_l1}
	
	The postulated similarity of finitely many moments as stated in the moment adaptation problem does not directly lead to the required error guarantees.
	The following Lemma, see Section~\ref{sec:details} for its proof, motivates the consideration of the stronger concept of similarity in $L^1$-difference.
	
	\begin{restatable}{lemmarep}{totalvariation}
	\label{lemma:motivation}
    Let $f\in\mathcal{F}$ and $p,q\in\mathcal{M}([0,1]^N)$ as defined in Section~\ref{sec:problem}. Then the following holds:
    \begin{align}
        \max_{l:[0,1]^N\to [0,1]} \left| \E_{q}\big[|f-l|\big] - \E_{p}\big[|f-l|\big] \right| = \frac{1}{2} \norm{p-q}_{L^1}.
    \end{align}
	\end{restatable}
	Lemma~\ref{lemma:motivation} shows that the $L^1$-difference between the densities $p$ and $q$ has to be small to achieve our goal.
	Assume the $L^1$-difference is not small, then there exists a labeling function $l_p:=l_q:=l$ such that the source risk $\E_p\big[|f-l_p|\big]$ is not a good indicator for the target risk $\E_q\big[|f-l_q|\big]$.
	Consequently, to achieve our goal, a small difference $\norm{\boldsymbol{\mu}_p-\boldsymbol{\mu}_q}_1$ between the moments has to imply a small $L^1$-difference.
	
	Unfortunately this is not the case without further conditions as even the uniform metric (which is smaller than the $L^1$-difference) can be very large for general densities with aligned moments only, see e.g.~\cite{lindsay2000moments}.
	
	\subsection{Convergence of High-Entropy Distributions}
	\label{subsec:high_entropy_distr}
	
	According to Subsection~\ref{subsec:from_moment_convergence_to_l1} additional assumptions on the densities are required to solve the moment adaptation problem.
	Therefore, we introduce a notion of $\epsilon$-close maximum entropy densities.
	We call a probability density {\it $\epsilon$-close maximum entropy density} if
	\begin{align}
	\label{eq:high_entropy_density}
	h_{\boldsymbol{\phi}}(p)-h(p)\leq\epsilon.
	\end{align}
	for some $\epsilon\geq 0$ and some vector $\boldsymbol{\phi}=(\phi_1,\ldots,\phi_{\psi(m,N)})^\text{T}$ of polynomials such that $1,\phi_1,\ldots,\phi_{\psi(m,N)}$ is a basis of $\mathbb{R}_m[x_1,\ldots,x_N]$.
	
	For some small $\epsilon$, by Eq.~\eqref{eq:maxent_distr_formula} and Pinsker's inequality, an $\epsilon$-close maximum entropy density $p$ fulfills $\norm{p-p^*}_{L^1}\leq \sqrt{2\epsilon}$ and can therefore be interpreted as being well approximable by its corresponding maximum entropy density $p^*$.
	In the language of Bayesian inference the term $D(p\Vert p^*)=h_{\boldsymbol{\phi}}(p)-h(p)$ measures the information gained when one revises one's beliefs from the prior probability density $p^*$ to the posterior probability density $p$.
	In this sense, the amount of information lost when using the moments $\int\boldsymbol{\phi} p$ instead of the density $p$ is at most $\epsilon$ for densities fulfilling Eq.~\eqref{eq:high_entropy_density}.
	Note that we allow $\epsilon$ to be zero to include maximum entropy densities $p=p^*$ in our discussions.
	The following Lemma~\ref{lemma:L1_convergence_in_M} (see Subsection~\ref{subsec:class} for its proof) motivates to consider $\epsilon$-close maximum entropy densities for tackling the moment adaptation problem defined in Section~\ref{sec:problem}.
	
	\begin{restatable}{lemmarep}{Lone}
		\label{lemma:L1_convergence_in_M}
		Consider some $\epsilon\geq 0$ and some vector $\phim=(\phi_1,\ldots,\phi_{\psi(m,N)})^\text{T}$ such that $1,\phi_1,\ldots,\phi_{\psi(m,N)}$ is a basis of $\mathbb{R}_m[x_1,\ldots,x_N]$ and let $p_n\in\M$ for $n\in\{1,\ldots,\infty\}$ be $\epsilon$-close maximum entropy densities with moments denoted by ${\boldsymbol\mu_n=\int \phim p_n}$.
		Then the following holds: 
		\begin{gather}
		\label{eq:convergence_in_M}
		\lim_{n\to\infty}\norm{\boldsymbol{\mu}_n-\boldsymbol{\mu}_\infty}_1=0
		\quad
		\implies
		\quad
		\limsup_{n\to\infty} \norm{p_n-p_\infty}_{L^1}\leq \sqrt{8 \epsilon}.
		\end{gather}
	\end{restatable}
	According to Eq.~\eqref{eq:simple_da_equation} a small misclassification risk in Eq.~\eqref{eq:misclassification_error} is implied by a small training error $\E_p[|f-l_p|]$, a small $L^1$-difference of the distributions and a small $\lambda^*$.
	According to Lemma~\ref{lemma:L1_convergence_in_M} this is the case if the densities $p,q\in\M$ have $\epsilon$-close maximum entropy and if the moment vectors $\int \phim p$ and $\int\phim q$ are similar.
	Unfortunately, the convergence in Eq.~\eqref{eq:convergence_in_M} can be very slow for sequences in $\M$ which is shown by the following example.
	
	\begin{example}
		\label{ex:truncated_normal}
		Consider the vector $\boldsymbol{\phi}_2=(x,x^2)^\text{T}\in\mathbb{R}_2[x]$ and two one-dimensional Truncated Normal Distributions with densities $p,q\in\M$ with equal variance but different means.
		These distributions are maximum entropy distributions constrained at the moments $\int \boldsymbol{\phi}_2 p$ and $\int\boldsymbol{\phi}_2 q$ and therefore satisfy Eq.~\eqref{eq:high_entropy_density} with $\epsilon=0$.
		It holds that for every moment difference $\norm{\int \phi_2 p-\int\phi_2 q}_1$ one can always find a small enough variance such that $\norm{p-q}_{L^1}$ is large.
	\end{example}
	
	Example~\ref{ex:truncated_normal} shows that additional properties besides Eq.~\eqref{eq:high_entropy_density} are required to obtain fast convergence rates for sequences in $\M$.
	
	\subsection{Convergence of Smooth High-Entropy Distributions}
	\label{subsec:smooth_high_entropy_distr}
	In this subsection we introduce additional smoothness conditions
	motivated by approximation results of exponential families~\cite{barron1991approximation} and Legendre polynomials~\cite{cox1988approximation}. More precisely, we consider the following set of densities.

		
	
	\begin{defi}
		\label{def:H}
    Let $\epsilon\geq 0$, $m\in\mathbb{N}$, $m\geq 2$ and $\phim=(\phi_1,\ldots,\phi_{m N})^\text{T}$ be a vector of polynomials such that $1,\phi_1,\ldots,\phi_{m N}$ is an orthonormal basis of $\mathrm{Span}(\mathbb{R}_m[x_1]\cup\ldots\cup\mathbb{R}_m[x_N])$.
    We call ${p\in\M}$ a \textit{smooth high-entropy density} iff the following three conditions are satisfied:\\
        $\phantom{\quad\quad}$\textbf{\emph{(A1)}} $h_{\phim}(p)-h(p)\leq \epsilon$\\
        $\phantom{\quad\quad}$\textbf{\emph{(A2)}} $\norm{\log p}_\infty \leq \frac{3m-6}{2}$\\
        $\phantom{\quad\quad}$\textbf{\emph{(A3)}} $\norm{\partial^m_{x_i} \log p_i}_{L^2}\leq 5^{m-4}\quad\forall i\in\{1,\ldots,N\}$,\\
    where $p_i=\int_0^1\cdots\int_0^1 p(x_1,\ldots,x_N)\diff x_1\cdots \diff x_{i-1} \diff x_{i+1}\cdots \diff x_d$ denote the marginal densities of $p$.
    We denote the set of all smooth high-entropy densities by $\mathcal{H}_{m,\epsilon}$.
	\end{defi}

	The set $\H_{m,\epsilon}$ in Definition~\ref{def:H} contains multivariate probability densities $p$ with loosely coupled marginals.
	The reason is the specification of the polynomial vector $\phim$ resulting in maximum entropy densities $p^*$ of densities $p\in\M$ with independent marginals (see Lemma~\ref{lemma:independence}).
	One advantage of this simplification is that no combinatorial explosion (curse of dimensionality) has to be taken into account.
	Such moment vectors have been shown to perform well in practice~\cite{zellinger2017central,zellinger2017robust} and distributions with loosely coupled marginals are created by many learning algorithms~\cite{comon1994independent,hyvarinen2001topographic,bach2002kernel}.
	Note that the present analysis can be extended to general multi-dimensional polynomial vectors by the usual product basis functions for polynomials.
	However, the use of such expansions is precluded by an exponential growth of the number of moments with the dimension $N$ and the consideration of additional smoothness constraints, see also~\cite{barron1991approximation}.
	
	The definition of the set $\mathcal{H}_{m,\epsilon}$ is independent of the choice of the orthonormal basis $1,\phi_1,\ldots,\phi_{m N}$.
	This follows from properties of the information projection~\cite{barron1991approximation}.
	
	The upper bounds on the $L^\infty$-norm and $L^2$-norm of the derivatives of the log-density functions restrict the smoothness of the densities.
	These bounds can be enlarged at the cost that more complicated dependencies on the shape of the log-density functions have to be considered in the subsequent analysis (see Subsection~\ref{subsec:convergence_of_smooth_functions}).
	It is interesting to observe that, when a density is bounded away from zero, assumptions on the log-densities are not too different from the assumptions on derivatives of the densities itself, see e.g.~\cite[Remark~2]{barron1991approximation}.
	
	$\H_{m,\epsilon}$ contains densities that are well approximable (in KL-divergence) by exponential families:
	For each $\epsilon>0$ and each density $p\in\M$ satisfying the smoothness constraints in $\mathcal{H}_{m,\epsilon}$ (i.e. log-density function bounded by $\frac{3m-6}{2}$ with derivative bounded in $L^2$-norm by $5^{m-1}$) there exists a number of moments $m$ such that $\min_{q\in\mathcal{E}} D(p\Vert q)\leq \epsilon$ for the exponential family $\mathcal{E}$ with sufficient statistic $\phim$.
	This follows from the fact that $h_{\phim}(p)\to h(p)$ for $m\to\infty$.
	
	The following Theorem~\ref{thm:bound_for_smooth_functions} (see Subsection~\ref{subsec:convergence_of_smooth_functions} for its proof) gives an uniform bound for the $L^1$-norm of the difference of densities in $\mathcal{H}_{m,\epsilon}$ in terms of differences of moments.
	
	\begin{restatable}{thmrep}{convergence}
		\label{thm:bound_for_smooth_functions}		
		Consider some $m$, $ \epsilon$, $\phim$ and $\mathcal{H}_{m,\epsilon}$ as in Definition~\ref{def:H} and let $p,q\in\mathcal{H}_{m,\epsilon}$ with moments denoted by $\boldsymbol{\mu}_p=\int\phim p$ and $\boldsymbol{\mu}_q=\int \phim q$. Then the following holds:
		\begin{gather*}
		\norm{\boldsymbol{\mu}_p-\boldsymbol{\mu}_q}_1 \leq \frac{1}{2 C \left(m+1\right)}\\
		\quad\implies\quad\\
		\norm{p-q}_{L^1} \leq \sqrt{2 C}\cdot \norm{\boldsymbol{\mu}_p-\boldsymbol{\mu}_q}_1 + \sqrt{8 \epsilon}
		\end{gather*}
		with the constant $C=2 e^{(3m-1)/2}$.
	\end{restatable}
	
	The more moments we consider in Theorem~\ref{thm:bound_for_smooth_functions}, i.e. the higher $m$ is, the richer is the class $\mathcal{H}_{m,\epsilon}$.
	However, with increasing $m$, the constant $C$ also increases.
	This constant depends exponentially on $m$ which is induced by the definition of the upper bounds on the norms of the derivatives in the Definition~\ref{def:H}.
	However, it is interesting to consider more general upper bounds $c_\infty\geq\norm{\log p}_\infty$ and $c_r\geq \norm{\partial^m_{x_i}\log p_i}_{L^2}$ instead.
	This leads to the constant $C$ as in Lemma~\ref{lemma:convergence_in_H} which depends double exponentially on the upper bounds $c_\infty$ and $c_r$.
	However, the double exponential dependency weakens when considering higher numbers $r$ of derivatives or numbers $m$ of moments (see Remark~\ref{remark:dependency}).
	Thus, the main influence is an exponential dependency on the upper log-density bound $c_\infty$.
	
	The considered dimension $N$ of the unit cube effects the number of moment differences considered in the $1$-norms in Theorem~\ref{thm:bound_for_smooth_functions}.
	By the specification of the vector $\phim$, this number increases only linearly with the dimension.
	
	Theorem~\ref{thm:bound_for_smooth_functions} together with Eq.~\eqref{eq:simple_da_equation} give a first result towards the goal of the moment adaptation problem: An upper bound on the misclassification risk of the discriminative model based on differences of moments:

	\begin{restatable}{correp}{corabsbound}
		\label{cor:bound_for_smooth_functions}
		Consider the set of high-entropy distributions $\mathcal{H}_{m,\epsilon}$ with $m$, $\epsilon$ and $\phim$ as in Definition~\ref{def:H}. Let $p,q\in\mathcal{H}_{m,\epsilon}$ with moments denoted by ${\boldsymbol{\mu}_p=\int\phim p}, {\boldsymbol{\mu}_q=\int \phim q}$ and let ${l_p,l_q:[0,1]^N\to [0,1]}$ be two labeling functions.
        Then the following holds for all $f:[0,1]^N\to\{0,1\}$:
		\begin{gather*}
		\norm{\boldsymbol{\mu}_p-\boldsymbol{\mu}_q}_1 \leq \frac{1}{2 C \left(m+1\right)}\\
		\quad\implies\quad\\
		\E_{q}\big[|f-l_{q}|\big] \leq \E_{p}\big[|f-l_{p}|\big] + \sqrt{2 C}\cdot \norm{\boldsymbol{\mu}_p-\boldsymbol{\mu}_q}_1 + \sqrt{8 \epsilon} + \lambda^*
		\end{gather*}
		with $C=2 e^{(3m-1)/2}$ and $\lambda^* = \inf_{h\in\mathcal{F}}\big(\E_p[|h-l_p|]+\E_q[|h-l_q|]\big)$.
	\end{restatable}
	Corollary~\ref{cor:bound_for_smooth_functions} gives an error bound on the target error that is relative to the error $\lambda^*$ of some optimal discriminative function.
	This is similar to the assumption in \textit{probably approximately correct learning theory} that there exists a perfect discriminative model in the underlying model class~\cite{quionero2009dataset}.
	The error $\lambda^*$ can be eliminated in the case of equal labeling functions, i.e. $l_p=l_q$, by using the bound of~\cite[Theorem~1]{ben2010theory} instead of Eq.~\eqref{eq:simple_da_equation}.

	Further implications of Corollary~\ref{cor:bound_for_smooth_functions} are discussed in more detail (together with the sample case) in Section~\ref{sec:main_result_on_learning_bounds}.
	
	\subsection{Relationship to other Probability Metrics}
	\label{subsec:contrib_to_picture_of_prob_metrics}
	
	Before stating our main result on learning bounds, let us establish an inequality relating the difference between moment vectors to the probability metrics considered in~\cite{gibbs2002choosing}, one of which being the L\'evy metric.
	
	\begin{defi}
	\label{def:levys_metric}
	    The L\'evy metric $d_L$ between two cumulative distribution functions $P,Q$ on the real line is defined by~\cite{levy1925probability}
	    \begin{align*}
	        d_\text{L}(P,Q) &= \inf{\left\{ \varepsilon\mid P(x-\varepsilon) - \varepsilon\leq Q(x)\leq P(x+\varepsilon) + \varepsilon, \forall x\in\mathbb{R} \right\}}.
	    \end{align*}
	\end{defi}
	The L\'evy metric assumes values in $[0,1]$, see e.g.~\cite{gibbs2002choosing}.
	
	\begin{restatable}{lemmarep}{lemmarachevbound}
	\label{lemma:rachev_bound}
	Let $m\in\mathbb{N}$ with $m\geq 2$, $\boldsymbol{\phi}\in\left(\mathbb{R}_m[x]\right)^K$ be a vector of moments with maximum degree $m$ and let $p,q\in\mathcal{M}([0,1])$ with cumulative distribution functions $P,Q$ and moments denoted by $\boldsymbol{\mu}_p:=\int \boldsymbol{\phi} p$ and $\boldsymbol{\mu}_q:=\int \boldsymbol{\phi} q$. Then there exist some constants $C_\text{L},M_L\in\mathbb{R}$ such that
	\begin{align}
	    \label{eq:rachev_inequality}
	    d_\text{L}(P,Q)\leq M_L\quad\implies\quad\norm{\boldsymbol{\mu}_p - \boldsymbol{\mu}_q}_1\leq C_\text{L}\cdot d_\text{L}(P,Q)^{\frac{1}{2 m+2}}.
	\end{align}
	\end{restatable}
	As a consequence of Eq.~\eqref{eq:rachev_inequality}, the value of $\norm{\boldsymbol{\mu}_p - \boldsymbol{\mu}_q}_1$ can be upper bounded by most other common probability metrics.
	
	Theorem~\ref{thm:bound_for_smooth_functions} upper bounds the $L^1$-difference between smooth high-entropy densities, or equivalently, upper bounds the total variation distance (see~\cite{gibbs2002choosing} for its definition and Lemma~\ref{lemma:motivation} for the equivalence proof).
	
	Figure~\ref{fig:metrics} shows how the herein applied moment-based metric relates to other probability metrics.
	
	\section{Main Result on Learning Bounds}
	\label{sec:main_result_on_learning_bounds}
	
	Theorem~\ref{thm:problem_solution} gives a first solution to the moment adaptation problem as described in Section~\ref{sec:problem}. Its proof is outlined in Subsection~\ref{subsec:problem_solution_proof}.
	
	\begin{restatable}{thmrep}{mainresult}
		\label{thm:problem_solution}
		Consider some $m$, $\epsilon$, $\phim$ and $\mathcal{H}_{m,\epsilon}$ as in Definition~\ref{def:H} and a function class $\mathcal{F}$ with finite VC-dimension $d$.
		Consider two probability densities $p,q\in\mathcal{H}_{m,\epsilon}$ and two (integrable) labeling functions ${l_p,l_q:[0,1]^N\to [0,1]}$.
		
		Let $X_p$ and $X_q$ be two arbitrary $k$-sized samples drawn from $p$ and $q$, respectively, and denote by $\widehat{\boldsymbol{\mu}}_p=\frac{1}{k}\sum_{\x\in X_p}\boldsymbol{\phi}_m(\x)$ and $\widehat{\boldsymbol{\mu}}_q=\frac{1}{k}\sum_{\x\in X_q}\boldsymbol{\phi}_m(\x)$ the corresponding sample moment vectors.
		
		Then, for every $\delta\in (0,1)$ and all $f\in\mathcal{F}$ the following holds with probability at least $1-\delta$ (over the choice of samples):
		If
		\begin{flalign}
		4 C^2(m+1)^2 m \delta^{-1} \leq k
		\end{flalign}
		and
		\begin{flalign}
		\norm{\widehat{\boldsymbol{\mu}}_p-\widehat{\boldsymbol{\mu}}_q}_1 \leq \left(2 (m+1) e C\right)^{-1}
		\end{flalign}
		then
		\begin{align}
		\label{eq:moment_adapt_result_bound}
		\begin{split}
		\E_{q}\big[|f-l_q|\big]\leq\, &\frac{1}{k}\sum_{\x\in X_p}|f(\x)-l_p(\x)| + 
		\sqrt{\frac{4}{k} \left( d\log \frac{2 e k}{d} + \log\frac{4}{\delta} \right)}+ \lambda^*\\
		&+ \sqrt{2 e C} \norm{\widehat{\boldsymbol{\mu}}_p-\widehat{\boldsymbol{\mu}}_q}_1 + \sqrt{8 C} \sqrt{\frac{N m}{k\delta}} + \sqrt{8\epsilon}
		\end{split}
		\end{align}
		where $C=2 e^{(3m-1)/2}$ and $\lambda^* = \inf_{h\in\mathcal{F}}\big(\E_p[|h-l_p|]+\E_q[|h-l_q|]\big)$.
	\end{restatable}
	
	
 Theorem~\ref{thm:problem_solution} directly extends the bounds on the target error (compare also Eq.~\eqref{eq:vapnik_bound}) in the statistical learning theory proposed by Vapnik and Chervonenkis~\cite{vapnik2015uniform} and the domain adaptation theory (compare also Eq.~\eqref{eq:simple_da_equation}) proposed by Ben-David et al.~\cite{ben2010theory} and gives a solution to the moment adaptation problem.
	
	Note that according to Vapnik and Chervonenkis~\cite{vapnik2015uniform}, a small misclassification risk of a discriminative model is induced by a small training error, if the sample size is large enough. Due to Ben-David et al.~\cite{ben2010theory}, this statement still holds for a test sample with a distribution different from the training sample, if the $L^1$-difference of the distributions is small and if there exists a model that can perform well on both distributions (error $\lambda^*$ in Eq.~\eqref{eq:simple_da_equation} is small).
	
	According to Theorem~\ref{thm:problem_solution}, a small misclassification risk of a model on a test sample with moments $\widehat{\boldsymbol{\mu}}_q$ is induced by a small error on a training sample with moments $\widehat{\boldsymbol{\mu}}_p$ being similar to $\widehat{\boldsymbol{\mu}}_q$, if the the following holds: The sample size is large enough, the densities $p$ and $q$ are smooth high-entropy densities with loosely coupled marginals, i.e. $p,q\in\mathcal{H}_{m,\epsilon}$, and there exists a model that can perform well on both densities.
	
	See Lemma~\ref{lemma:sample_convergence_in_H} in Subsection~\ref{subsec:problem_solution_proof} for improved assumptions and an improved constant $C$ with the drawback of some additional and more complicated assumptions on the smoothness of the densities.
	
	It is interesting to investigate in more detail the terms in Eq.~\eqref{eq:moment_adapt_result_bound} that depend on the sample size $k$ (chosen equally for both samples for better readability):
	Let us therefore assume a fixed number of moments $m$ and a given probability ${1-\delta}$.
	For model classes with VC-dimension $d\geq N$ (i.e. supra-linear models) and for a large sample size $k > d$, the complexity of the proposed term is bounded by $O(\sqrt{d/k})$ which is smaller than the complexity $O(\sqrt{d/k \log( 2ek/d)})$ of the classical error bound in the first line of Eq.~\eqref{eq:moment_adapt_result_bound} as proposed in~\cite{vapnik2015uniform}.
	However, the classical term decreases faster with complexity $O(\sqrt{\log(1/\delta)})$ as the probability $1-\delta$ decreases compared to the proposed term which decreases only with complexity $O(\sqrt{1/\delta})$.
	
	\section{Application to Unsupervised Domain Adaptation}
	\label{sec:application}
	
	In the following, we show how to analyze the generalization ability of moment-based algorithms as proposed in~\cite{zellinger2017central,zellinger2017robust,nikzad2018domain,peng2018cross,ke2018identity,Wei2018GenerativeAG,xing2018adaptive} for the problem of unsupervised domain adaptation under the covariate shift assumption (Subsection~\ref{subsec:domain_adaptation}).
	
	Therefore, let us consider an open set $\Omega\subseteq\mathbb{R}^M$, two densities $p,q\in\mathcal{M}(\Omega)$, a labeling function $l:\Omega\to [0,1]$ (covariate shift), a $k$-sized sample $X_p$ drawn from $p$ with labels $Y_p=l(X_p)$ and an unlabeled $k$-sized sample $X_q$ drawn from $q$ as defined in Subsection~\ref{subsec:domain_adaptation}.
	
	The considered approaches search for a function $g:\Omega\to [0,1]^N\in \mathcal{G}$ and a function $f:[0,1]^N\to\{0,1\}\in\mathcal{F}$ such that the differences of finitely many sample moments of the mapped samples $g(X_p)$ and $g(X_q)$ are similar and such that the model $f\circ g$ has a small misclassification risk on the sample $X_p$.
	This is done by minimizing the following objective function:
	\begin{align}
	\label{eq:objective}
	\min_{f\in\mathcal{F},g\in\mathcal{G}} \frac{1}{k}\sum_{\x\in X_p} |f(g(\x))-l(\x)| + d_m(g(X_p),g(X_q))
	\end{align}
	where
	\begin{align*}
	d_m(X,X')=\sum_{j=1}^{m} \norm{c_j(X)-c_j(X')]}_2
	\end{align*}
	is the Central Moment Discrepancy regularizer~\cite{zellinger2017central,zellinger2017robust} with empirical expectation vector $c_1(X)=\E[X]=\frac{1}{k}\sum_{\x\in X} \x$ and sampled central moment $c_j(X)=\E[(X-c_1(X))^j]$ where $\x^j$ denotes element-wise power.
	The term $d_m(g(X_p),g(X_q))$ in Eq.~\eqref{eq:objective} is a simple aggregation of finitely many differences of sampled central moments~\cite{zellinger2017robust} from the marginal densities of $p$ and $q$.
	
    Our example is based on a function class $\mathcal{F}$ with finite VC-dimension $d$ and the function class
    \begin{align}
	\label{eq:G}
	\mathcal{G} &=\{g\in C^r(\Omega,[0,1]^N)\mid r\geq M-N+1, \mathrm{rank}\, \mathbf{J}_g=N\,\text{a.e.}\},
	\end{align}
    where $C^r(\Omega,[0,1]^N)$ refers to the set of functions $g:\Omega\to[0,1]^N$ with continuous derivatives up to order $r$, $\mathrm{rank}\, \mathbf{J}_g$ refers to the rank of the Jacobian matrix $\vec J_g$ of the function $g$ and $a.\,e.$~abbreviates {\it almost everywhere}.
    The definition of $\mathcal{G}$ in Eq.~\eqref{eq:G} together with the openness of $\Omega$ ensures that the pushforward measures $\mu\circ g^{-1}$ and $\nu\circ g^{-1}$ of two Borel probability measures $\mu$ and $\nu$ with densities $p$ and $q$, respectively, have probability densities $\tilde{p}$ and $\tilde{q}$, respectively, see~\cite{ponomarev1987submersions} for a proof.
    
    Consider some $\epsilon\geq 0$ and let the maximum order of moments be $m=5$ as it is appropriate for many practical tasks, see e.g.~\cite{zellinger2017central,zellinger2017robust,peng2018cross,ke2018identity,xing2018adaptive,peng2019weighted,Wei2018GenerativeAG}.
    Let us further denote by
    \begin{align}
        \label{eq:phim_application}
        \phim=\left(\eta_1(x_1),\ldots,\eta_5(x_1),\eta_1(x_2),\ldots,\eta_5(x_2),\ldots,\eta_1(x_N),\ldots,\eta_1(x_N),\ldots,\eta_5(x_N)\right)^\text{T}
    \end{align}
    the vector of polynomials such that
    \begin{align*}
        &\eta_1(x)=\sqrt{3} (2 x-1)\\
        &\eta_2(x)=\sqrt{5} \left(6 x^2-6 x+1\right)\\
        &\eta_3(x)=\sqrt{7} \left(20 x^3-30 x^2+12 x-1\right)\\
        &\eta_4(x)=3 \left(70 x^4-140 x^3+90 x^2-20 x+1\right)\\ &\eta_5(x)=\sqrt{11} \left(252 x^5-630 x^4+560 x^3-210 x^2+30 x-1\right)
    \end{align*}
    are the orthonormal Legendre polynomials in the variable $x$ up to order $5$.
    
    Let $g\in\mathcal{G}$ be such that the latent densities fulfill
    \begin{align*}
    h_{\phim}(\tilde p)-h(\tilde p)\leq \epsilon\quad\text{and}\quad h_{\phim}(\tilde q)-h(\tilde q)\leq \epsilon
    \end{align*}
    and have log-density functions $\log \tilde p, \log \tilde q\in W_2^5$ such that
    \begin{align*}
    \norm{\log\tilde p}_\infty\leq 5, \norm{\log \tilde q}_\infty\leq 5\quad\text{and}\quad\norm{\partial^5_{x_i}\log \tilde p_i}\leq 10, \norm{\partial^5_{x_i}\log \tilde q_i}\leq 10
    \end{align*}
    for all $i\in\{1,\ldots,N\}$.
    
    Following~\cite{ben2007analysis}, we define the labeling function ${l_p:[0,1]^N\to [0,1]}$ by
    \begin{align*}
    l_p(\vec a)=\frac{\int_{\{\x\mid g(\x) =\vec a\}} l(\x) p(\x) \diff\x }{\int_{\{\x\mid g(\x) =\vec a\}} p(\x)\diff\x}
    \end{align*}
    and $l_q$ analogously.
    Let the sample size $k\geq 6.3\cdot 10^9$ and $\norm{\widehat{\boldsymbol{\mu}}_p-\widehat{\boldsymbol{\mu}}_q}_1\leq 2.3\cdot 10^{-5}$ (or equivalently $d_m(g(X_p),g(X_q))\leq 6.7\cdot 10^{-12}$) with $\widehat{\boldsymbol{\mu}}_p=\frac{1}{k}\sum_{\x\in g(X_p)}\boldsymbol{\phi}_m(\x)$ and $\widehat{\boldsymbol{\mu}}_q=\frac{1}{k}\sum_{\x\in g(X_q)}\boldsymbol{\phi}_m(\x)$ denoting the corresponding sample moment vectors for $\boldsymbol{\phi}_m$ as in Eq.~\eqref{eq:phim_application}.
    Then, by applying Theorem~\ref{thm:problem_solution} on the domains $(\tilde p,l_p)$ and $(\tilde q,l_q)$ with the improved assumptions and constants of Lemma~\ref{lemma:sample_convergence_in_H}, the following holds with probability at least $0.8$:
    \begin{align}
    \label{eq:ex_wo_cmd}
    \begin{split}
    \int \left|f-l_q\right|\tilde q \leq\, &\frac{1}{k}\sum_{\x \in X_p}\left|f(g(\x))-l(\x)\right| + 
    \sqrt{\frac{4}{k} \left( d\log \frac{2 e k}{d} + 3 \right)}+ \lambda^*\\
    &+ 84.6 \norm{\widehat{\boldsymbol{\mu}}_p-\widehat{\boldsymbol{\mu}}_q}_1 + 513 \sqrt{\frac{N}{k}} + \sqrt{8\epsilon}.
    \end{split}
    \end{align}
    If $\rho_{i j}=\E[X_p]$ and $\nu_{i j}=\E[X_q]$ denote the $i$-th empirical raw moments of $\tilde p$ and $\tilde q$ in the variable $x_j$, then
	\begin{align}
	\norm{\widehat{\boldsymbol{\mu}}_p-\widehat{\boldsymbol{\mu}}_q}_1 &= \sum_{j=1}^{N} \sum_{i=1}^{5} \left\lvert \E[\eta_i(X_p)]-\E[\eta_i(X_q)]\right\rvert\nonumber\\
	&\leq C_5 \cdot \sum_{j=1}^{N} \sum_{i=1}^{5}\left\lvert \rho_{i j}-\nu_{i j}\right\rvert\nonumber\\
	&\leq C_5\cdot \sum_{j=1}^{N} \sum_{i=1}^{5} \sum_{t=0}^{i} \binom{i}{t}  \left\lvert \rho_{t j}' \rho_{1 j}^{i-t}-\nu_{t j}' \nu_{1 j}^{i-t}\right\rvert\nonumber\\
	&\leq C_5\cdot \sum_{j=1}^{N} \sum_{i=1}^{5} \sum_{t=0}^{i} \binom{i}{t}
	\left(\left\lvert \rho_{1 j}^{i-t}-\nu_{1 j}^{i-t}\right\rvert+\left\lvert \rho_{t j}'-\nu_{t j}' \right\rvert\right)\nonumber\\
	\nonumber
	&\leq C_5\cdot \sum_{j=1}^{N} \sum_{i=1}^{5} \sum_{t=0}^{i} \binom{i}{t}
	\left(\left(i-t\right) \left\lvert \rho_{1 j}-\nu_{1 j}\right\rvert+\left\lvert \rho_{t j}'-\nu_{t j}' \right\rvert\right),
	\end{align}
    where $C_5=\max_{i\in\{1,\ldots,N\}} r_i$ and $r_i=\sum_{t=1}^5 |l_t|$ is the sum of the absolute values of the coefficients $l_t$ of all terms in the orthonormal Legendre polynomials $\eta_1(x_j),\ldots,\eta_5(x_j)$ which contain the monomial $x_j^i$.
    The term $\rho_{i j}'=\E[(X_j-E[X_j])^i], i\in\mathbb{N}$ denotes the $i$-th sampled central moment of the marginal density $p_j$, especially $\rho_{0 j}'=1$ and $\rho_{1 j}'=0$.
    The terms $\nu_{i j}'$ analogously denote the sampled central moments of the marginal densities of $q$.
    The second inequality follows from the Binomial theorem, the third inequality follows from the fact that
	\begin{align*}
	|x_1 y_1-x_2 y_2|\leq |x_1-x_2|+|y_1-y_2|\quad\forall x_1,x_2,y_1,y_2\in[-1,1]
	\end{align*}
	and the fourth inequality follows from
    \begin{align*}
	|x_1^k-x_2^k|\leq k\cdot |x_1-x_2|\quad\forall x_1,x_2\in[-1,1], k\in\mathbb{N}.
	\end{align*}
	It further holds that
	\begin{align}
    \norm{\widehat{\boldsymbol{\mu}}_p-\widehat{\boldsymbol{\mu}}_q}_1 &\leq C_5\cdot \sum_{j=1}^{N} \sum_{i=1}^{5} \sum_{t=0}^{i} \binom{i}{t}
	\left(\left(i-t\right) \left\lvert \rho_{1 j}-\nu_{1 j}\right\rvert+\left\lvert \rho_{t j}'-\nu_{t j}' \right\rvert\right)\nonumber\\
	&\leq C_5\cdot \sum_{j=1}^{N} \sum_{i=1}^{5} \sum_{t=0}^{i} i \binom{i}{t}
	\left(\left\lvert \rho_{1 j}-\nu_{1 j}\right\rvert+\left\lvert \rho_{t j}'-\nu_{t j}' \right\rvert\right)\nonumber\\
	&\leq C_5\cdot \sum_{j=1}^{N} \sum_{i=1}^{5} \sum_{t=0}^{5} 5 \binom{5}{t}
	\left(\left\lvert \rho_{1 j}-\nu_{1 j}\right\rvert+\left\lvert \rho_{t j}'-\nu_{t j}' \right\rvert\right)\nonumber\\
	&\leq C_5\cdot 5^2\cdot \max_{t\in\{0,1,\ldots,5\}} \left\{\binom{5}{t}\right\}\cdot \sum_{j=1}^{N} \sum_{t=0}^{5} \left(\left\lvert \rho_{1 j}-\nu_{1 j}\right\rvert+\left\lvert \rho_{t j}'-\nu_{t j}' \right\rvert\right)\nonumber\\
	\label{eq:lone_cmd_bound}
	&\leq C_5\cdot 5^2\cdot (5+1)\cdot \max_{t\in\{0,1,\ldots5\}} \left\{\binom{5}{t}\right\} \cdot\sqrt{N}\cdot d_5(p,q),
	\end{align}
	where the last inequality follows from $\norm{\x}_2\leq \sqrt{N}\cdot \norm{\x}_1$.
	From the ``change of variables'' Theorem~4.1.11 in~\cite{dudley2002real} we obtain
    \begin{align}
    \label{eq:change_of_vars}
        \int \left| f-l_q \right| \tilde q
        =\int \left| f-l_q \right| \diff(Q\circ g^{-1})
        = \int \left|f-l_q\right|\circ g\diff Q
        = \int \left| f\circ g-l\right| q.
    \end{align}
    In particular, if the dimension of the latent space is taken to be $N=5$, the sample size $k= 6.3\cdot 10^9$ and if the function class $\mathcal{F}$ is the class of neural networks with one layer, $5$ nodes and signum activation function for each node, i.e.~the VC-dimension is $d=6$, then the following holds by Eq.~\eqref{eq:change_of_vars}, Eq.~\eqref{eq:lone_cmd_bound} and Eq.~\eqref{eq:ex_wo_cmd}, with probability at least $0.8$:
    \begin{align*}
    \begin{split}
    \int \left|f\circ g-l\right| q \leq\, &\frac{1}{k}\sum_{\x_\in X_p}\left|f(g(\x))-l(\x)\right| + 2.96\cdot 10^8\cdot d_5({X_p,X_q}) + 0.0148 + \sqrt{8\epsilon} + \lambda^*,
    \end{split}
    \end{align*}
    where the error originating from the application of statistical learning theory is approximately $2.95\cdot 10^{-4}$ and the sampling error originating from our analysis is approximately $1.44\cdot 10^{-2}$.


	
	
	
	\section{Proofs}
	\label{sec:details}
	
	All proofs are summarized in this section together with additional remarks and comments.
	\subsection{Proofs of Subsection~\ref{subsec:from_moment_convergence_to_l1} on Moment Similarity and $L^1$-Similarity}
	
	\totalvariation*
	
	\begin{proof}
	    Let us define the labeling function ${l^*:[0,1]^N\to [0,1]}$ by
	    \begin{align}
	    l^*(\x)=
	        \begin{cases}
               1~\text{if}~f(\x)=1~\text{and}~p(\x) < q(\x)\\
               1~\text{if}~f(\x)=0~\text{and}~p(\x)\geq q(\x)\\
               0~\text{if}~f(\x)=1~\text{and}~p(\x)\geq q(\x)\\
               0~\text{if}~f(\x)=0~\text{and}~p(\x) < q(\x)
             \end{cases}
	    \end{align}
	    By this construction the following holds:
	    \begin{align}
	    \label{eq:lstar_identity}
	        |f-l^*|=\mathbbm{1}_A
	    \end{align} where $\mathbbm{1}_A(\x)=\begin{cases}1:\x\in A\\0:\text{else}\end{cases}$ and $A:=\{\x\in [0,1]^N\mid p(\x)\geq q(\x)\}$.
	    From Eq.~\eqref{eq:lstar_identity} we obtain
	    \begin{align}
	    \label{eq:lstar_identity_integral}
	    \begin{split}
	        \int_{[0,1]^N} |f-l^*|\, (p-q) &=
	        \int_{[0,1]^N} \mathbbm{1}_A (p-q)\\
	        &= \int_{[0,1]^N} \mathbbm{1}_A\, p - \int_{[0,1]^N} \mathbbm{1}_A\, q\\
	        &= 1 - \int_{[0,1]^N} \mathbbm{1}_{A^c}\, p - 1 + \int_{[0,1]^N} \mathbbm{1}_{A^c}\, q\\
	        &= \int_{[0,1]^N} \mathbbm{1}_{A^c} (q-p)
	        \end{split}
	    \end{align}
	    where $A^c:=[0,1]^N\setminus A$ denotes the complement of $A$.
	    
	    For all $l:[0,1]^N\to [0,1]$, it holds that
	    \begin{align*}
	        \left| \E_{q}\big[|f-l|\big] - \E_{p}\big[|f-l|\big] \right| &= \left| \int_{[0,1]^N} |f-l|\, (p-q) \right|\\
	        &\leq \left|\sup_{\x\in[0,1]^N} \big\{|f(\x)-l(\x)|\big\} \int_{[0,1]^N} (p-q) \right|\\
	        &\leq \left| \int_{[0,1]^N} (p-q) \right|\\
	        &\leq \max\left\{ \int_{[0,1]^N} (p-q), \int_{[0,1]^N} (q-p) \right\}\\
	        &\leq \max\left\{ \int_{[0,1]^N}\mathbbm{1}_A\, (p-q), \int_{[0,1]^N}\mathbbm{1}_{A^c}\, (q-p) \right\}\\
	        &= \int_{[0,1]^N} |f-l^*|\, (p-q)
	    \end{align*}
	    where the last line is obtained from Eq.~\eqref{eq:lstar_identity_integral}.
	    It follows that
	    \begin{align*}
	        \sup_{l:[0,1]^N\to [0,1]} \left| \E_{q}\big[|f-l|\big] - \E_{p}\big[|f-l|\big] \right|\leq \int_{[0,1]^N} |f-l^*|\, (p-q).
	    \end{align*}
	    Since $l^*:[0,1]^N\to [0,1]$, it also holds that
	    \begin{align*}
	        \int_{[0,1]^N} |f-l^*|\, (p-q) &\leq \sup_{l:[0,1]^N\to [0,1]} \left| \E_{q}\big[|f-l|\big] - \E_{p}\big[|f-l|\big] \right|
	    \end{align*}
	    and therefore
	    \begin{align}
	        \max_{l:[0,1]^N\to [0,1]} \left| \E_{q}\big[|f-l|\big] - \E_{p}\big[|f-l|\big] \right| = \int_{[0,1]^N} |f-l^*|\, (p-q).
	    \end{align}
	    Using Eq.~\eqref{eq:lstar_identity} and Eq.~\eqref{eq:lstar_identity_integral} yields
	    \begin{align*}
	        2 \int_{[0,1]^N} |f-l^*|\, (p-q) &= 2 \int_{[0,1]^N}\mathbbm{1}_A\, (p-q)\\
	        &= \int_{[0,1]^N}\mathbbm{1}_A\, (p-q) + \int_{[0,1]^N}\mathbbm{1}_{A^c}\, (q-p)\\
	        &= \int_{[0,1]^N}|p-q|
	    \end{align*}
	    which finalizes the proof.\qed
	\end{proof}

	\subsection{Proofs of Subsection~\ref{subsec:high_entropy_distr} on the Convergence of High-Entropy Distributions}
	\label{subsec:class}
	
	For the rest of this subsection consider some $\phim=(\phi_1,\ldots,\phi_{\psi(m,N)})^\text{T}$ such that $1,\phi_1,\ldots,\phi_{\psi(m,N)}$ is a basis of $\mathbb{R}_m[x_1,\ldots,x_N]$.
	We further consider the set $\M$ of probability distributions on the unit cube, the differential entropy $h$, the maximum entropy $h_{\boldsymbol{\phi}}$ and the KL-divergence $D(.\Vert.)$ as defined in Section~\ref{sec:preliminaries}.
	We denote by $p^*$ the maximum entropy density of some $p\in\M$ constrained at the moments $\int \phim p$.
	
	The following Lemma~\ref{lemma:basic_inequality} provides a key relationship allowing to focus on differences of distributions in exponential families.
	
	\begin{restatable}{lemmarep}{basicinequ}
		\label{lemma:basic_inequality}
		Consider some $\epsilon\geq 0$ and some $p, q\in\M$ having $\epsilon$-close maximum entropy.
		Then the following holds:
		\begin{align}
		\label{eq:basic_epsilon_bound}
		\norm{p-q}_{L^1} \leq \sqrt{2 D(p^*\Vert q^*)} + \sqrt{8 \epsilon}.
		\end{align}
	\end{restatable}
	
	\begin{proof}
		Applying the Triangle Inequality and Pinsker's Inequality yields
		\begin{align*}
		\norm{p-q}_{L^1} &\leq \norm{p^*-q^*}_{L^1}
		+ \norm{p^*-p}_{L^1} + \norm{q^*-q}_{L^1}\\
		&\leq \sqrt{2 D(p^*\Vert q^*)}
		+ \sqrt{2 D(p\Vert p^*)} + \sqrt{2 D(q\Vert q^*)}.
		\end{align*}
		The exponential form of the maximum entropy distribution Eq.~\eqref{eq:maxent_distr_formula} implies that $D(p\Vert p^*)=h_{\boldsymbol{\phi}}(p)-h(p)$ and therefore $D(p\Vert p^*)\leq\epsilon$ such that Eq.~\eqref{eq:basic_epsilon_bound} follows.\qed
	\end{proof}
	
	Lemma~\ref{lemma:convergence_in_M} analyzes the convergence in KL-divergence of sequences of distributions in exponential families in terms of the convergence of respective moment vectors.
	
	\begin{restatable}{lemmarep}{convinM}
		\label{lemma:convergence_in_M}		
		Let $(p_n)_{n\in\mathbb{N}}\subset \mathcal{M}([0,1]^N)$ and $p_\infty\in\mathcal{M}([0,1]^N)$ such that $p_n$ is an $\epsilon$-close maximum entropy density for all $n\in\{1,\ldots,\infty\}$ and denote its respective moments by $\boldsymbol\mu_n=\int \phim p_n$.
		Then the following holds: 
		\begin{gather*}
		\lim_{n\to\infty}\norm{\boldsymbol{\mu}_n-\boldsymbol{\mu}_\infty}_1=0
		\quad
		\implies
		\quad
		\lim_{n\to\infty}D(p_n^*\Vert p_\infty^*)=0.
		\end{gather*}
	\end{restatable}
	
	\begin{proof}
		The maximum entropy density $p_n^*$ of $p_n$ is independent of the choice of the basis $1,\phi_1,\ldots,\phi_{\psi(m,N)}$~\cite{barron1991approximation}.
		Therefore, we may assume without loss of generality that the elements of $\phim$ are solely positive monomials.
		
		According to Eq.~\eqref{eq:maxent_distr_formula}, the maximum entropy distributions $p_n^*$ are of the form $p^*_n=c(\boldsymbol\lambda_n) \exp\left(-\langle\boldsymbol\lambda_n,\phim\rangle\right)$ with parameter vectors $\boldsymbol\lambda_n\in\mathbb{R}^{\psi(m,N)}$.
		Using Eq.~\eqref{eq:KL_divergence} and the fact that $p_n^*\in\M$ yields
		\begin{align*}
		D(p^*_n\Vert p^*_\infty) &=\int p_n^* \log \frac{p_n^*}{p_\infty^*}\\
		&=\int p_n^* \log\frac{c(\boldsymbol{\lambda}_n) \exp(-\langle\boldsymbol{\lambda}_n,\phim\rangle)}{c(\boldsymbol{\lambda}_\infty) \exp(-\langle\boldsymbol{\lambda}_\infty,\phim\rangle)}\\
		&=\int p_n^* (\log c(\boldsymbol{\lambda}_n) - \log c(\boldsymbol{\lambda}_\infty)) + \int p_n^* (-\langle\boldsymbol{\lambda}_n,\phim\rangle + \langle\boldsymbol{\lambda}_\infty,\phim\rangle)\\
		&= \left(\log c(\boldsymbol{\lambda}_n) - \log c(\boldsymbol{\lambda}_\infty)\right) + (-\langle\boldsymbol{\lambda}_n,\int p_n^* \phim\rangle + \langle\boldsymbol{\lambda}_\infty,\int p_n^*\phim\rangle)\\
		&= \left(\log c(\boldsymbol{\lambda}_n) - \log c(\boldsymbol{\lambda}_\infty)\right) + (-\langle\boldsymbol{\lambda}_n,\boldsymbol{\mu}_n\rangle + \langle\boldsymbol{\lambda}_\infty,\boldsymbol{\mu}_n\rangle)\\
		&= \left(\log c(\boldsymbol{\lambda}_n) - \log c(\boldsymbol{\lambda}_\infty)\right) + \langle\boldsymbol{\mu}_n,\boldsymbol{\lambda}_\infty-\boldsymbol{\lambda}_n\rangle\\
		&\leq \left|\log c(\boldsymbol\lambda_n)-\log c(\boldsymbol\lambda_\infty)\right| + \left\langle\boldsymbol\mu_n,|\boldsymbol\lambda_n-\boldsymbol\lambda_\infty|\right\rangle\\
		&\leq \left|\log c(\boldsymbol\lambda_n)-\log c(\boldsymbol\lambda_\infty)\right| + \norm{\boldsymbol\lambda_n-\boldsymbol\lambda_\infty}_1.
		\end{align*}
		where the last inequality follows from the choice of the basis $1,\phi_1,\ldots,\phi_{\psi(m,N)}$.
		
		In the following we show that $\log c(\boldsymbol\lambda_n)\to \log c(\boldsymbol\lambda_\infty)$ and $\boldsymbol\lambda_n\to\boldsymbol\lambda_\infty$ as $\boldsymbol\mu_n\to\boldsymbol\mu_\infty$:
		The elements of the parameter vector $\boldsymbol\lambda_*$ of the maximum entropy distribution $p^*=c(\boldsymbol\lambda_*) \exp\left(-\langle\boldsymbol\lambda_*,\phim\rangle\right)$ in Eq.~\eqref{eq:maxent_distr_formula} correspond to the Lagrange multipliers solving the optimization problem $\min_{\boldsymbol\lambda\in\mathbb{R}^{\psi(m,N)}}\Gamma(\boldsymbol\lambda) $ where $\Gamma(\boldsymbol\lambda)=\langle \boldsymbol\lambda, \boldsymbol\mu_*\rangle-\log(c(\boldsymbol\lambda))$
		and $\boldsymbol\mu_p=\int\phim p=(\int \phi_1 p,\ldots,\phi_{\psi(m,N)} p)^\text{T}$, see e.g.~\cite{agmon1979algorithm,batou2013calculation,wainwright2008graphical}.
		Let $q=c(\boldsymbol\lambda_q)\exp\left(-\langle\boldsymbol\lambda_q,\phim\rangle\right)$ be a probability density of an exponential family with moments $\boldsymbol\mu_q=\int\phim q$ and parameter vector $\boldsymbol{\lambda}_q:=(\lambda_1,\ldots,\lambda_{\psi(m,N)})^\text{T}$.
		Then the partial derivative of the function $\boldsymbol\lambda_q\mapsto\Gamma(\boldsymbol\lambda_q)$ \wrt~the variable $\lambda_i$ is given by
		\begin{align*}
		    \partial_{\lambda_i} \Gamma(\boldsymbol{\lambda}_q) &= \int \phi_i p - \partial_{\lambda_i} \log c(\boldsymbol{\lambda}_q)\\
		    &= \int \phi_i p - \frac{1}{c(\boldsymbol{\lambda}_q)}\partial_{\lambda_i} c(\boldsymbol{\lambda}_q)\\
		    &= \int \phi_i p + \frac{\partial_{\lambda_i} \int \exp\left(-\langle \boldsymbol{\lambda}_q,\phim\rangle\right) }{c(\boldsymbol{\lambda}_q) \left(\int \exp\left(-\langle \boldsymbol{\lambda}_q,\phim\rangle\right)\right)^{2}} \\
		    &= \int \phi_i p + c(\boldsymbol{\lambda}_q) \int \exp\left(-\langle \boldsymbol{\lambda}_q,\phim\rangle\right) ( - \partial_{\lambda_i} \langle \boldsymbol{\lambda}_q,\phim\rangle)\\
		    &= \int \phi_i p - \int c(\boldsymbol{\lambda}_q) \exp\left(-\langle \boldsymbol{\lambda}_q,\phim\rangle\right) \phi_i\\
		    &= \int \phi_i p - \int \phi_i q
		\end{align*}
		and the gradient vector $\nabla \Gamma(\boldsymbol\lambda_q)$ can therefore be computed by
		\begin{align*}
		    \nabla \Gamma(\boldsymbol\lambda_q) &= \boldsymbol\mu_p - \int \phim q.
		\end{align*}
		Consequently, the second partial derivative \wrt~the variables $\lambda_i$ and $\lambda_j$ is given by
		\begin{align*}
		    \partial_{\lambda_i,\lambda_j}^2 \Gamma(\boldsymbol{\lambda}_q) &= \partial_{\lambda_j} (\int \phi_i p - \int \phi_i q)\\
		    &= \int c(\boldsymbol{\lambda}_q) \exp\left(-\langle \boldsymbol{\lambda}_q,\phim\rangle\right) \phi_i \phi_j - \int \exp\left(-\langle \boldsymbol{\lambda}_q,\phim\rangle\right) \phi_i\, \partial_{\lambda_j} c(\boldsymbol{\lambda}_q)\\
		    &= \int \phi_i\phi_j q - \int \exp\left(-\langle \boldsymbol{\lambda}_q,\phim\rangle\right) \phi_i\, c(\boldsymbol{\lambda}_q)^2 \int \exp\left(-\langle \boldsymbol{\lambda}_q,\phim\rangle\right)\phi_j\\
		    &= \int \phi_i\phi_j q - \int q\phi_i (\int q \phi_j)
		\end{align*}
		and the Hessian matrix $H_\Gamma(\boldsymbol\lambda_q)$ can be computed by
		\begin{align*}
		H_\Gamma(\boldsymbol\lambda_q) &= \int (\phim\cdot \phim^\text{T}) q - \int \phim q\cdot(\int\phim q)^\text{T}.
		\end{align*}
		The Hessian matrix $H_\Gamma$ equals the covariance matrix of a random variable with density $q$. It is assumed that the elements of $\phim$ are independent. $H_\Gamma$ is therefore positive definite and the function $\boldsymbol\lambda_q\mapsto \Gamma(\boldsymbol\lambda_q)$ reaches its minimum at a vector with $\nabla\Gamma(\boldsymbol\lambda_*)=0$, especially at $\boldsymbol\lambda_*$.
		The Implicit Function Theorem can be applied to the function $I:(\boldsymbol\mu,\boldsymbol\lambda)\mapsto\boldsymbol\mu - \int \phim c(\boldsymbol\lambda)\exp\left(-\langle\boldsymbol\lambda,\phim\rangle\right)$ guaranteeing the existence af an open set $U\subset\mathbb{R}^{\psi(m,N)}$ (containing $\boldsymbol{\mu}_p$) and a unique continuous function $g:\boldsymbol\mu\mapsto\boldsymbol\lambda$ with $I(\boldsymbol\mu,g(\boldsymbol\mu))=0$ for all $\boldsymbol\mu\in U$.
		Consequently the convergence of the moment vector $\boldsymbol\mu_n\to\boldsymbol\mu_\infty$ implies the convergence of the corresponding parameter vectors $\boldsymbol\lambda_n\to\boldsymbol\lambda_\infty$ as $n$ tends to infinity.

        The convergence of $\log c(\boldsymbol\lambda_n)$ to $\log c(\boldsymbol\lambda_\infty)$ follows from the continuity of the \textit{cumulant function} $\boldsymbol\lambda \mapsto -\log c(\boldsymbol\lambda) = \log \left(\int \exp\left(-\langle \boldsymbol{\lambda},\phim\rangle\right)\right)$, see e.g.~\cite[Proposition~3.1]{wainwright2008graphical}.\qed
	\end{proof}
	Lemma~\ref{lemma:basic_inequality} together with Lemma~\ref{lemma:convergence_in_M} motivate to focus on densities with $\epsilon$-close maximum entropy and together prove Lemma~\ref{lemma:L1_convergence_in_M}.
	
	In the following Subsection~\ref{subsec:approximation_theory} we recall additional properties on the densities, such that fast convergence rates can be obtained.
	
	%
	
	\subsection{Preliminaries from Approximation Theory}
	\label{subsec:approximation_theory}
	
	Smoothness conditions on densities appropriate for our goal are established in~\cite{barron1991approximation} and~\cite{cox1988approximation}.
	The following serves as a key lemma.
	
	\begin{restatable}{lemmarep}{barronandshoulemmafive}
		\label{lemma:barron_and_sheu_lemma5}
		Consider some $\phim=(\phi_1,\ldots,\phi_m)^\text{T}$ such that $1,\phi_1,\ldots,\phi_m$ is a basis of $\mathbb{R}_m[x]$ orthonormal with respect to some probability density $q$ for which $\norm{\log q}_\infty<\infty$ and consider some $A_q\in\mathbb{R}$ such that $\norm{f_m}_\infty\leq A_q \norm{f_m}_{L^2(q)}$ for all $f_m\in \mathbb{R}_m[x]$.
		Consider some vector of moment values $\boldsymbol{\mu}\in [0,1]^m$.
		
		Let $p_0\in\mathcal{M}([0,1])$ and denote by $\boldsymbol{\mu}_0=\int \phim p_0$ its moments and by $p_0^*$ the corresponding maximum entropy density.
		Let further $b=e^{\norm{\log q/p_0^*}_\infty}$.
		
		If
		\begin{align}
		\label{eq:valid_moment_diff}
		\norm{\boldsymbol{\mu}-\boldsymbol{\mu}_0}_2\leq \frac{1}{4 A_q e b}
		\end{align}
		then the maximum entropy density $p^*\in\M$ fulfilling $\int\phi_m p^*=\boldsymbol{\mu}$ exists and satisfies
		\begin{align}
		\label{eq:log_ration}
		\norm{\log{p_0^*}/{p^*}}_\infty &\leq 4 e^t b A_q \norm{\boldsymbol{\mu}-\boldsymbol{\mu}_0}_2 \leq t\\
		\label{eq:barron_inequ}
		D(p_0^*\Vert p^*) &\leq 2 e^t b \norm{\boldsymbol{\mu}-\boldsymbol{\mu}_0}_2^2
		\end{align}
		for $t$ satisfying $4 e b A_q \norm{\boldsymbol{\mu}-\boldsymbol{\mu}_0}_2\leq t\leq 1$ with Euler's number $e$.
	\end{restatable}
	
	\begin{proof}
		See~\cite[Lemma~5]{barron1991approximation}.\qed
	\end{proof}
	The following Corollary~\ref{lemma:barron_proof_thm3} follows from Lemma~\ref{lemma:barron_and_sheu_lemma5} and shows the relation between results on the approximation by exponential families and results on the approximation by polynomials.
	
	\begin{restatable}{correp}{barronproofthmthree}
		\label{lemma:barron_proof_thm3}
		Consider some polynomial vector $\phim=(\phi_1,\ldots,\phi_m)^\text{T}$ such that $1,\phi_1,\ldots,\phi_m$ is an orthonormal basis of $\mathbb{R}_m[x]$.
		
		Let $p\in\mathcal{M}([0,1])$ such that $\log p\in W_2^r$ with the Sobolev space $W_2^r$ and some $A_p\in\mathbb{R}$ such that $\norm{f_m}_\infty\leq A_p \norm{f_m}_{L^2(p)}$ for all $f_m\in \mathbb{R}_m[x]$.
		
		Denote by $f=\log p$ and by $p^*$ the maximum entropy density constrained at the moments $\int \phim p$. Further denote by $\gamma=\min_{f_m\in\mathbb{R}_m[x]} \norm{f-f_m}_\infty$ and $\xi=\min_{f_m\in\mathbb{R}_m[x]} \norm{f-f_m}_{L^2(p)}$ minimal errors of approximating $f$ by polynomials $f_m\in \mathbb{R}_m[x]$.
		Then the following holds:
		\begin{align*}
		4 e^{4\gamma + 1} A_p \xi \leq 1\quad\implies\quad
		\norm{\log{p/p^*}}_\infty \leq 2 \gamma + 4 e^{4 \gamma + 1} \xi A_p.
		\end{align*}
	\end{restatable}
	
	\begin{proof}
		See the first part of the proof of Theorem~3 in~\cite{barron1991approximation}.\qed
	\end{proof}
	The following Corollary gives some insights in the case of maximum entropy densities constrained at sample moments.
	
	\begin{restatable}{correp}{barronsample}
	    \label{cor:barron_sample}
	    Let $p,\phim,A_p,\gamma,\xi$ as in Corollary~\ref{lemma:barron_proof_thm3}, $b:=e^{2 \gamma + 4 e^{4 \gamma + 1} \xi A_p}$ and denote by $\widehat{\boldsymbol{\mu}}_p=\frac{1}{k}\sum_{\x\in X_p}\phim(\x)$ the sample moments of a $k$-sized sample $X_p$ drawn from $p$.
	    
	    If $4 e^{4\gamma + 1} A_p \xi \leq 1$ then for all $\delta\in (0,1)$ such that $(4 e b A_p)^2 m \leq \delta k$ with probability at least $1-\delta$ the maximum entropy density $\widehat p$ constrained at the moments $\widehat{\boldsymbol{\mu}}_p$ exists and the following holds:
	   \begin{align}
	   \label{eq:barron_sample1}
	        D(p^*\Vert \widehat{p})\leq 2 e b \frac{m}{k\delta}\\
	        \label{eq:barron_sample2}
	        \norm{\log p/\widehat{p}}_\infty\leq 1
	    \end{align}
	\end{restatable}
	
	\begin{proof}
		For the proof of Eq.~\eqref{eq:barron_sample1} see the second part of the proof of Theorem~3 in~\cite{barron1991approximation}.
		The proof of Eq.~\eqref{eq:barron_sample2} follows immediately by applying the full Lemma~5, i.e. including Eq.(5.7), of~\cite{barron1991approximation} in the same proof of Theorem~3 in~\cite{barron1991approximation}.\qed
	\end{proof}
	
	Note that the approximation error $\xi$ in Corollary~\ref{lemma:barron_proof_thm3} is in terms of $L^2(p)$-norm instead of $L^2(\nu)$ with uniform weight function $\nu$.
	To obtain concrete values for the constant $A_p$ in Corollary~\ref{lemma:barron_proof_thm3}, the following result can be applied.
	
	\begin{restatable}{lemmarep}{barronAbound}
		\label{lemma:barron_A_bound}
		Consider a polynomial $f_m\in\mathbb{R}_m[x]$ with degree less than or equal to $m$ on $[0,1]$. Then the following holds:
		\begin{align}
		\norm{f_m}_\infty \leq (m+1) \norm{f_m}_{L^2}
		\end{align}
	\end{restatable}
	
	\begin{proof}
		See e.g.~\cite[Lemma~6]{barron1991approximation}.\qed
	\end{proof}
	The following result from the theory of approximation by orthonormal polynomials can be used to obtain concrete values for the approximation errors $\gamma$ and $\xi$ in Corollary~\ref{lemma:barron_proof_thm3}.
	
		
	
	\begin{restatable}{lemmarep}{coxlemma}
\label{lemma:cox}
Consider some $m\geq r\geq 2$ and some $f\in W_2^r$ with Sobolev space $W_2^r$.
Let further $\sum_{i=0}^\infty b_i \phi_i(x)=f(x)$ be the representation of $f$ by real numbers $b_0,b_1,b_2,\ldots\in\mathbb{R}$ and normalized Legendre polynomials $\phi_0,\phi_1,\phi_2,\ldots\in\mathbb{R}[x]$ which are assumed to be orthonormal for the scalar product $\langle \phi_i,\phi_j\rangle :=\int_0^1 \phi_i(x)\phi_j(x)\diff x$.
Denote by $f_k:=\sum_{i=0}^k b_i \phi_i$. Then the following holds:
\begin{align}
    \label{eq:cox_infty_bound}
    \norm{f-f_m}_\infty &\leq \frac{e^r \norm{f^{(r)}}_2}{2^r \sqrt{r-1} (m+r)^{r-1}}\\
    \label{eq:cox_L2_bound}
    \norm{f-f_m}^2_2 &\leq \frac{\norm{f^{(r)}}^2_2}{4^r (m+r+1)\cdots (m-r+2)} 
\end{align}
\end{restatable}
	
	\begin{proof}
		See~\cite{cox1988approximation}.\qed
	\end{proof}
	
	\subsection{Proofs of Subsection~\ref{subsec:smooth_high_entropy_distr} on the Convergence of Smooth High-Entropy Distributions}
	\label{subsec:convergence_of_smooth_functions}
	
	In this Subsection, we propose a uniform upper bound on the $L^1$-difference between two densities in the set $\mathcal{H}_{m,\epsilon}$ (see Definition~\ref{def:H}), that is linear in terms of the $\ell^1$-norm of the difference of finite moment vectors.
	Let us start with the following helpful statement.
	
	\begin{restatable}{lemmarep}{lemmaApbound}
	   \label{lemma:Ap_bound}
	   Let $f_m\in\mathbb{R}_m[x]$ be a polynomial of degree less than or equal to $m$ on $[0,1]$ and $p\in\mathcal{M}([0,1])$ such that $\norm{\log p}_\infty= c_\infty$ for some $c_\infty\in\mathbb{R}$. Then the following holds:
	   \begin{align*}
	       \norm{f_m}_\infty \leq (m+1) e^{c_\infty/2} \norm{f_m}_{L^2(p)}.
	   \end{align*}
	\end{restatable}
	
	\begin{proof}
		For all $f_m\in\mathbb{R}_m[x]$ the following holds by  Lemma~\ref{lemma:barron_A_bound}:
		\begin{align*}
		\norm{f_m}_\infty &\leq (m+1)\norm{f_m}_{L^2}\\
		&=(m+1)\sqrt{\int_0^1 |f_m|^2 \frac{p}{p}}\\
		&\leq (m+1)\sqrt{\sup\frac{1}{|p|} \int_0^1 |f_m|^2 p}
		\end{align*}
		Since ${c_\infty=\norm{\log p}_\infty}$, it holds that $-c_\infty\leq \log p\leq c_\infty$ and therefore also ${e^{-c_\infty}\leq 1/|p|\leq e^{c_\infty}}$ which yields the required result.\qed
	\end{proof}
	The following Lemma~\ref{lemma:convergence_in_H} serves as our anchor in the approximation theory recalled in Subsection~\ref{subsec:approximation_theory}.
	
	\newcommand{\ga}{\gamma}
	\newcommand{\de}{\xi}
	\begin{restatable}{lemmarep}{lemmaconinH}
		\label{lemma:convergence_in_H}
		Consider some $m\geq r\geq 2$ and some $\phim=(\phi_1,\ldots,\phi_m)^\text{T}$ such that $1,\phi_1,\ldots,\phi_m$ is an orthonormal basis of $\mathbb{R}_m[x]$.
		
		Let $p,q\in\mathcal{M}([0,1])$ such that $\log p,\log q\in W_2^r$ with Sobolev space $W_2^r$ and denote by $p^*$ and $q^*$ corresponding maximum entropy densities constrained at the moments $\boldsymbol{\mu}_p=\int\phim p$ and $\boldsymbol{\mu}_q=\int \phim q$.
		
		If $4 e^{4\ga + 1} e^{c_\infty/2} (m+1) \de\leq 1$ then the following holds:
		\begin{gather}
		\norm{\boldsymbol{\mu}_p-\boldsymbol{\mu}_q}_2 \leq \frac{1}{2 C \left(m+1\right)}
		\quad\implies\quad
		D(p^*\Vert q^*) \leq C\cdot \norm{\boldsymbol{\mu}_p-\boldsymbol{\mu}_q}_2^2
		\end{gather}
		where
		\begin{align}
		\label{eq:ct}
		C &= 2 e^{1+c_\infty+2 \gamma + 4 e^{4 \gamma + 1} \xi e^{\nicefrac{c_\infty}{2}}(m+1)}
		\end{align}
		and
		\begin{align}
		\label{eq:gamma}
		\ga &= \frac{e^r}{\sqrt{r-1} (m+r)^{r-1}} \left(\frac{1}{2}\right)^r c_r\\
		\label{eq:xi}
		\de^2 &= \frac{e^{c_\infty}}{(m+r+1)\cdots (m-r+2)}\left( \frac{1}{4} \right)^r c_r^2\\
		c_r &= \norm{\partial^r_x \log p}_{L^2}\\
		c_\infty &= \norm{\log p}_\infty.
		\end{align}
	\end{restatable}
	
	\begin{remark}
		\label{remark:dependency}
		For $\gamma$ and $\xi$ as defined in Lemma~\ref{lemma:convergence_in_H} it holds that $\gamma,\xi\in O\left(\frac{1}{m^{r-1}}\right)$ and therefore $C\to 2 e^{1+c_\infty}$ as ${m,r\to\infty}$.
	\end{remark}
	
	\begin{proof}
		Let $m,r$ be such that $m\geq r\geq 2$.
		Consider some $\phim=(\phi_1,\ldots,\phi_m)^\text{T}$ with $1,\phi_1,\ldots,\phi_m$ forming an orthonormal basis of $\mathbb{R}_m[x]$, i.e. forming an orthonormal basis of $\mathbb{R}_m[x]$ \wrt~the uniform weight function $\tilde q$ which is $1$ if $x\in[0,1]$ and $0$ otherwise.
		For $\tilde q$ it holds that $\norm{\log\tilde q}_\infty < \infty$ and with $A_{\tilde q}:=m+1$, due to Lemma~\ref{lemma:barron_A_bound}, it also holds that
		\begin{align*}
		\norm{f_m}_\infty \leq A_{\tilde q} \norm{f_m}_{L^2(\tilde q)}
		\end{align*}
		for all  $f_m\in\mathbb{R}_m[x]$.
		Let $p,q\in\mathcal{M}([0,1])$ such that $\log p,\log q\in W_2^r$ and denote its moments by $\boldsymbol{\mu}_p=\int\phim p$ and $\boldsymbol{\mu}_q=\int \phim q$.
		Choose $\tilde{\boldsymbol{\mu}}:=\boldsymbol{\mu}_q$, $\tilde p_0:=p, \tilde b:=e^{\norm{\log \tilde q/\tilde p_0^*}_\infty}$ and note that
		\begin{align*}
		\tilde b:=e^{\norm{\log \tilde q/\tilde p_0^*}_\infty}
		= e^{\norm{\log \tilde q/p^*}_\infty}
		= e^{\norm{\log \tilde q-\log p*}_\infty}
		= e^{\norm{\log p^*}_\infty}.
		\end{align*}
		If
		\begin{align*}
		\norm{\boldsymbol{\mu}_p-\boldsymbol{\mu}_q}_2 \leq \frac{1}{4 (m+1) e^{1+\norm{\log p^*}_\infty}}
		\end{align*}
		then, due to Lemma~\ref{lemma:barron_and_sheu_lemma5}, the maximum entropy density $q^*\in\mathcal{M}([0,1])$ satisfies
		\begin{align*}
		D(p^*\Vert q^*)\leq 2 e^{\tilde t} e^{\norm{\log p^*}_\infty} \norm{\boldsymbol{\mu}_p-\boldsymbol{\mu}_q}_2^2
		\end{align*}
		for $\tilde t$ satisfying $4 (m+1) e^{1+\norm{\log p^*}_\infty} \norm{\boldsymbol{\mu}_p-\boldsymbol{\mu}_q}_2\leq \tilde t\leq 1$, in particular for $\tilde t=1$, such that
		\begin{align*}
		D(p^*\Vert q^*)\leq 2 e^{1+\norm{\log p^*}_\infty}\cdot \norm{\boldsymbol{\mu}_p-\boldsymbol{\mu}_q}_2^2.
		\end{align*}
		In the following, we aim at an upper bound on $\norm{\log p^*}_\infty$. It holds that
		\begin{align}
		\label{eq:log_p_in_proof}
		\norm{\log p^*}_\infty=\norm{\log p^* p/p}_\infty = \norm{\log p - \log p/p^*}_\infty \leq \norm{\log p}_\infty + \norm{\log p/p^*}_\infty
		\end{align}
		where the last inequality is due to the Triangle Inequality.
		Lemma~\ref{lemma:Ap_bound} yields
		\begin{align}
		\norm{f_m}_\infty &\leq (m+1) e^{c_\infty/2} \norm{f_m}_{L^2(p)}.
		\end{align}
		Denote by $\tilde p:=p$, $\tilde f:=\log p$ and $A_{\tilde p}:=(m+1) e^{c_\infty/2}$. Further denote by $\tilde \gamma:=\min_{f_m\in\mathbb{R}_m[x]} \norm{f-f_m}_\infty$ and ${\tilde \xi:=\min_{f_m\in\mathbb{R}_m[x]} \norm{f-f_m}_{L^2(p)}}$ minimal errors of approximating $f$ by polynomials $f_m\in \mathbb{R}_m[x]$.
		From Corollary~\ref{lemma:barron_proof_thm3}, we obtain
		\begin{align*}
		4 e^{4\tilde\gamma + 1} A_{\tilde p} \tilde\xi \leq 1\quad\implies\quad
		\norm{\log{p/p^*}}_\infty \leq 2 \tilde\gamma + 4 e^{4 \tilde\gamma + 1} \tilde\xi A_{\tilde p}.
		\end{align*}
		Consider $\gamma$ and $\xi$ as defined in Eq.~\eqref{eq:gamma} and Eq.~\eqref{eq:xi}, respectively.
		Lemma~\ref{lemma:cox} yields $\tilde \gamma\leq \gamma, \tilde \xi\leq \xi$ and therefore also
		\begin{align*}
		4 e^{4\tilde\gamma + 1} A_{\tilde p} \tilde\xi\leq 4 e^{4\ga + 1} e^{c_\infty/2} (m+1) \de
		\end{align*}
		Consequently, if $4 e^{4\ga + 1} e^{c_\infty/2} (m+1) \de\leq 1$ then
		\begin{align}
		\label{eq:proof_bridge}
		\norm{\log{p/p^*}}_\infty \leq  2 \gamma + 4 e^{4 \gamma + 1} \xi (m+1) e^{c_\infty/2}
		\end{align}
		and together with Eq.~\eqref{eq:log_p_in_proof} we obtain
		\begin{align*}
		\norm{\log p^*}_\infty \leq c_\infty + 2 \gamma + 4 e^{4 \gamma + 1} \xi (m+1) e^{c_\infty/2}.
		\end{align*}
		\qed
	\end{proof}
	To obtain simpler statements and useful bounds for small moment orders, we consider specific upper bounds on the norms of the log-derivatives as defined in Definition~\ref{def:H} of the set $\in \mathcal{H}_{m,\epsilon}$ of smooth high-entropy densities.
	
	
	
	\begin{restatable}{lemmarep}{lemmasimplify}
		\label{lemma:simplification}
		Consider some $\epsilon\geq 0$, $m=r\geq 2$ and let $p\in\mathcal{H}_{m,\epsilon}$. Then the following holds:
		\begin{align*}
		4 e^{4\ga + 1} e^{c_\infty/2} (m+1) \de\leq 1\quad\text{and}\quad C\leq 2 e^{(3m-1)/2},
		\end{align*}
		where $\gamma$, $\xi$, $c_r$, $c_\infty$, $C$ are defined as in Lemma~\ref{lemma:convergence_in_H}.
	\end{restatable}
	
	\begin{proof}
		We start by proving the following inequalities inductively for $m\geq 2$ with $m\in\mathbb{N}$:
		\begin{align}
		\label{eq:induction1}
		5^{m-4} &\leq\frac{\sqrt{(2 m+1)!\, (m-1)}}{2 e^{m+2}}\\
		\label{eq:induction2}
		\frac{3m-6}{2} &\leq \log\left( \frac{e^m 2^{m-1}}{(m+1) \sqrt{m-1}} \right)\\
		\label{eq:induction3}
		\frac{\sqrt{(2 m+1)!}}{4^m m^{m-1} e^{2}}&\leq \frac{1}{4}
		\end{align}
		For ${m=2,\ldots,7}$ all inequalities are fulfilled.
		Note that for any $m\geq 8$ the non-negativeness of later considered terms is ensured. To continue our proof by induction we may therefore assume that Eqs.~\eqref{eq:induction1}--\eqref{eq:induction3} are fulfilled for some arbitrary but fixed $m\in\mathbb{N}$ with $m\geq 8$.
        
        Since
        $$
        \frac{(2 m+3)(2m +2) m}{(m-1)} - 25 e^2
        $$ is a positive and monotonic increasing sequence for $m\geq 8$  (as can be proven with any computer algebra system), it follows that
        \begin{align*}
            5\leq \sqrt{\frac{(2\cdot 8+3)(2\cdot 8 +2) 8}{(8-1) e^2}}\leq 
            \sqrt{\frac{(2 m+3)(2m +2) m}{(m-1) e^2}}
        \end{align*}
        such that
        \begin{align*}
            5^{m+1-4} &=5^{m-4}\cdot 5\\
            &\leq 5^{m-4}\sqrt{\frac{(2 m+3)(2m +2) m}{(m-1) e^2}}\\
            &\leq \frac{\sqrt{(2 m+1)!\, (m-1)}}{2 e^{m+2}} \sqrt{\frac{(2 m+3)(2m +2) m}{(m-1) e^2}}\\
            &= \frac{\sqrt{(2 (m+1)+1)!\, ((m+1)-1)}}{2 e^{(m+1)+2}}.
        \end{align*}
        
        Since
        \begin{align*}
            \log\left(\frac{e\, 2\, (m+1)\sqrt{m-1}}{(m+2)\sqrt{m}}\right) - \frac{3}{2}
        \end{align*}
        is a positive and monotonic increasing sequence for $m\geq 8$  (as can be proven with any computer algebra system), it follows that
        \begin{align*}
            \frac{3}{2}\leq \log\left(\frac{e\, 2\, (8+1)\sqrt{8-1}}{(8+2)\sqrt{8}}\right)
            \leq \log\left(\frac{e\, 2\, (m+1)\sqrt{m-1}}{(m+2)\sqrt{m}}\right)
        \end{align*}
        such that
        \begin{align*}
            \frac{3 (m+1)-6}{2} &= \frac{3m-6}{2} + \frac{3}{2}\\
            &\leq \frac{3m-6}{2} + \log\left(\frac{e\, 2\, (m+1)\sqrt{m-1}}{(m+2)\sqrt{m}}\right)\\
            &\leq \log\left( \frac{e^m 2^{m-1}}{(m+1) \sqrt{m-1}} \right) + \log\left(\frac{e\, 2\, (m+1)\sqrt{m-1}}{(m+2)\sqrt{m}}\right)\\
            &=\log\left( \frac{e^{(m+1)} 2^{(m+1)-1}}{((m+1)+1) \sqrt{(m+1)-1}} \right).
        \end{align*}
        
        Since
        \begin{align*}
            \frac{\sqrt{(2 m+3)(2 m+2)}}{4 m} - 1
        \end{align*}
        is a negative and monotonic decreasing sequence for $m\geq 8$  (as can be proven with any computer algebra system), it follows that
        \begin{align*}
            \frac{\sqrt{(2 m+3)(2 m+2)}}{4 m}\leq \frac{\sqrt{(2\cdot 8+3)(2\cdot 8+2)}}{4\cdot 8}\leq 1
        \end{align*}
        such that
        \begin{align*}
            \frac{\sqrt{(2 (m+1)+1)!}}{4^{(m+1)} (m+1)^{(m+1)-1} e^{2}} &\leq
            \frac{\sqrt{(2 m+1)!}}{4^m m^{m-1} e^{2}} \frac{\sqrt{(2 m+3)(2 m+2)}}{4 m}\\
            &\leq \frac{\sqrt{(2 m+1)!}}{4^m m^{m-1} e^{2}}\cdot 1\\
            &\leq \frac{1}{4}.
        \end{align*}
		
		According to Definition~\ref{def:H} and the verified Eq.~\eqref{eq:induction1} it holds that
		\begin{align}
		    \label{eq:cr_proof}
		    c_r &\leq 5^{m-4} \leq \frac{\sqrt{(2 m+1)!\, (m-1)}}{2 e^{m+2}}
		\end{align}
		which, together with Eq.~\eqref{eq:induction3}, implies that
		\begin{align}
		\label{eq:gamma_proof}
		    \ga &\leq \frac{e^m}{\sqrt{m-1} (2 m)^{m-1}} \left(\frac{1}{2}\right)^m \frac{\sqrt{(2 m+1)!\, (m-1)}}{2 e^{m+2}}=
		    \frac{\sqrt{(2 m+1)!}}{4^m m^{m-1} e^{2}}\leq \frac{1}{4}.
		\end{align}
		Applying Eq.~\eqref{eq:gamma_proof}, the definition of $\xi$ and Eq.~\eqref{eq:cr_proof}, we obtain
		\begin{align*}
		    4 e^{4\ga + 1} &e^{c_\infty/2} (m+1) \de\\
		    &\leq
		    4 e e (m+1) e^{c_\infty/2} \frac{e^{c_\infty/2}}{\sqrt{(2 m+1)!}}\left( \frac{1}{2} \right)^m \frac{\sqrt{(2 m+1)!\, (m-1)}}{2 e^{m+2}}\\
		    &=
		    e\frac{4 e (m+1) \sqrt{(m-1)}}{2 e^{m+2} 2^m} e^{c_\infty}
	    \end{align*}
	    From Definition~\ref{def:H} and Eq.~\eqref{eq:induction2} we know that
		\begin{align}
		    \label{eq:cinfty_proof}
		    c_\infty &\leq \frac{3 m -6}{2} \leq \log\left( \frac{e^m 2^{m-1}}{(m+1) \sqrt{m-1}} \right)
		\end{align}
		which further gives
	    \begin{align*}
		    e\frac{4 e (m+1) \sqrt{(m-1)}}{2 e^{m+2} 2^m} e^{c_\infty} \leq e\frac{4 e (m+1) \sqrt{(m-1)}}{2 e^{m+2} 2^m} \left( \frac{e^m 2^{m-1}}{(m+1) \sqrt{m-1}} \right)=1
		\end{align*}
		and therefore
		\begin{align*}
		    4 e^{4\ga + 1} &e^{c_\infty/2} (m+1) \de\leq 1.
		\end{align*}
		From Eq.~\eqref{eq:ct} we obtain
		\begin{align*}
		    C &= 2 e^{1+c_\infty+2 \gamma + 4 e^{4 \gamma + 1} \xi e^{\nicefrac{c_\infty}{2}}(m+1)}
		    \leq 2 e^{2 + 2\ga + c_\infty}
		\end{align*}
		and by applying Eq.~\eqref{eq:cinfty_proof} and Eq.~\eqref{eq:gamma_proof} it holds that
		\begin{align*}
		    C \leq 2 e^{\frac{5}{2} + \frac{3 m -6}{2}}
		    \leq 2 e^{\frac{3 m -1}{2}}.
		\end{align*}\qed
	\end{proof}
	The following lemma allows to focus on distributions from exponential families with independent marginals by considering specific vectors of polynomials.
	
		
	
		

\begin{restatable}{lemmarep}{lemmaindependence}
		\label{lemma:independence}
		Consider some polynomial vector $\phim=(\phi_1,\ldots,\phi_{m N})^\text{T}$ such that $1,\phi_1,\ldots,\phi_{m N}$ is an orthonormal basis of ${\mathrm{Span}()}$.
		
		 Let $p^*,q^*$ be two maximum entropy densities constrained at the moments $\int\phim p,\int\phim q$ for some $p,q\in\M$. 
		 Then the following holds:
		 \begin{align}
		    D(p^*\Vert q^*)=\sum_{i=1}^N D(p_i^*\Vert q_i^*)
		\end{align}
		where $p_i^*$ denotes the maximum entropy density of $p$ constrained at the moments $\int \boldsymbol{\phi}_m^{(i)} p$ for some vector $\boldsymbol{\phi}_m^{(i)}=(\phi_{i1},\ldots,\phi_{im})$ such that $1,\phi_{i1},\ldots,\phi_{im}$ is an orthonormal basis of $\mathbb{R}_m[x_i]$.
	\end{restatable}
	
	\begin{proof}
	    According to Eq.~\eqref{eq:maxent_distr_formula} it holds that $p_i^*$ is of the form
	    \begin{align*}
	        p_i^*(x_i) = c_i(\boldsymbol\lambda_i) \exp\left(-\langle \boldsymbol\lambda_i,\boldsymbol{\phi}_m^{(i)}(x_i)\rangle\right)
    	\end{align*}
    	where $c_i(\boldsymbol\lambda_i)=\left(\int_0^1 \exp\left(-\langle \boldsymbol\lambda_i,\boldsymbol{\phi}_m^{(i)}(x_i)\rangle\right)dx_i\right)^{-1}$ is the constant of normalization and $\boldsymbol\lambda_i\in\mathbb{R}^{m}$ is a parameter vector.
	    It follows that
	    \begin{align*}
	        \tilde p^* &:= p_1^*\cdots p_N^*\\
	        &= c_1(\boldsymbol\lambda_1) \exp\left(-\langle \boldsymbol\lambda_1,\boldsymbol{\phi}_m^{(1)}(x_1)\rangle\right) \cdots c_N(\boldsymbol\lambda_N) \exp\left(-\langle \boldsymbol\lambda_N,\boldsymbol{\phi}_m^{(N)}(x_N)\rangle\right)\\
	        &= \left(\int_{[0,1]^N} \exp\left(-\langle \boldsymbol{\tilde\lambda},\boldsymbol{\tilde\phi}_m(\x)\rangle\right)d\x\right)^{-1} \exp\left(-\langle \boldsymbol{\tilde\lambda},\boldsymbol{\tilde\phi}_m(\x)\rangle\right)
	    \end{align*}
	    where $\boldsymbol{\tilde\lambda}\in\mathbb{R}^{m N}$ is the concatenation of the vectors $\boldsymbol\lambda_1,\ldots,\boldsymbol\lambda_N$ and $\boldsymbol{\tilde\phi}_m\in\mathbb{R}_m[x_1,\ldots,x_N]$ is the vector of polynomials obtained as the concatenation of $\boldsymbol{\phi}_m^{(1)},\ldots,\boldsymbol{\phi}_m^{(N)}$.
	    It holds that $\tilde p^*$ is a probability density of exponential form with sufficient statistic $\boldsymbol{\tilde\phi}_m$.
	    The elements of $\boldsymbol{\tilde\phi}_m$, together with the unit $1$, form an orthonormal basis of ${\mathrm{Span}()}$.
	    The uniqueness and the exponential form of the maximum entropy density $p^*$ implies that $\tilde p^*=p^*$ and the following holds:
	    \begin{align*}
	        D(p^*\Vert q^*)&=\int_{[0,1]^N} p^*\log\frac{p^*}{q^*}d\x\\
	        &=\int_0^1\ldots\int_0^1 p_1^*\cdots p_N^*\log\frac{p_1^*\cdots p_N^*}{q_1^*\cdots q_N^*} dx_1\ldots dx_N\\
	        &= \sum_{i=1}^N \int_0^1\ldots\int_0^1 p_1^*\cdots p_N^* \log\frac{p_i^*}{q_i^*}dx_1\ldots dx_N\\
	       &= \sum_{i=1}^N\left( \int_0^1 p_i^*\log\frac{p_i^*}{q_i^*} dx_i\prod_{j\neq i} \int_0^1 p_j^* dx_j\right)\\
	        &= \sum_{i=1}^N \int_0^1 p_i^*\log\frac{p_i^*}{q_i^*} dx_i\\
	        &= \sum_{i=1}^N D(p_i^*\Vert q_i^*).
	    \end{align*}\qed
	\end{proof}
	We are now ready to prove Theorem~\ref{thm:bound_for_smooth_functions}.
	
	\convergence*
	
	\begin{proof}
	    Consider some $m, \epsilon,\phim$ and $\mathcal{H}_{m,\epsilon}$ as in Definition~\ref{def:H} and some $p, q\in \mathcal{H}_{m,\epsilon}$.
	    Then, $p, q\in \mathcal{M}([0,1]^N)$ and have $\epsilon$-close maximum entropy.
	    Applying Lemma~\ref{lemma:basic_inequality} yields
		\begin{align*}
		\norm{p-q}_{L^1} \leq \sqrt{2 D(p^*\Vert q^*)} + \sqrt{8 \epsilon}
		\end{align*}
		for $p^*,q^*$ being the maximum entropy densities constrained at the moments $\int\phim p,\int\phim q$.
		The vector $\phim=(\phi_1,\ldots,\phi_{m N})^\text{T}$ is a polynomial vector such that $1,\phi_1,\ldots,\phi_{m N}$ is an orthonormal basis of ${\mathrm{Span}()}$. 
	    Therefore, by applying Lemma~\ref{lemma:independence}, we obtain
	    \begin{align}
	    \label{eq:inequ_impl_proof_thm1}
	        \norm{p-q}_{L^1} \leq \sqrt{2\sum_{i=1}^N D(p_i^*\Vert q_i^*)} + \sqrt{8 \epsilon},
	    \end{align}
	    where $p_i^*$ denotes the maximum entropy density of $p$ constrained at the moments $\int \boldsymbol{\phi}_m^{(i)} p$ for some vector $\boldsymbol{\phi}_m^{(i)}=(\phi_{i1},\ldots,\phi_{im})$ such that $1,\phi_{i1},\ldots,\phi_{im}$ is an orthonormal basis of $\mathbb{R}_m[x_i]$.
	    
	    The densities $p_i^*$ can also be seen as maximum entropy densities constrained at the moments $\boldsymbol{\mu}_{p_i}:=\int_0^1 \boldsymbol{\phi}_m^{(i)} p_i$ for the marginal densities $p_i$ of $p$ defined by
	    \begin{align*}
	        p_i(x_i):=\int_0^1\cdots\int_0^1 p(x_1,\ldots,x_N)\, d x_1\cdots d x_{i-1} d x_{i+1}\cdots d x_N.
	    \end{align*}
	    From Definition~\ref{def:H} it follows that $\norm{\partial_{x_i}^m \log p_i}_{L^2}\leq 5^{m-4}$ and therefore ${\log p_i\in W_2^m}$ with Sobolev space $W_2^m$.
	    If $4 e^{4\ga + 1} e^{c_\infty/2} (m+1) \de\leq 1$ the following holds by Lemma~\ref{lemma:convergence_in_H}:
		\begin{gather}
		\label{eq:inequ_impl_proof_thm1_2}
		\norm{\boldsymbol{\mu}_{p_i}-\boldsymbol{\mu}_{q_i}}_2 \leq \frac{1}{2 C \left(m+1\right)}
		\quad\implies\quad
		D(p_i^*\Vert q_i^*) \leq C\cdot \norm{\boldsymbol{\mu}_{p_i}-\boldsymbol{\mu}_{q_i}}_2^2
		\end{gather}
		with $C,\ga,c_\infty,\de$ as defined in Lemma~\ref{lemma:convergence_in_H} with $r=m$.
	    Since $p\in\H_{m,\epsilon}$, Lemma~\ref{lemma:simplification} implies that $4 e^{4\ga + 1} e^{c_\infty/2} (m+1) \de\leq 1$ and $C\leq e^{(3m-6)/2}$.
	    Since
	    \begin{align*}
	        \norm{\boldsymbol{\mu}_{p_i}-\boldsymbol{\mu}_{q_i}}_2 \leq 
	        \norm{\boldsymbol{\mu}_{p}-\boldsymbol{\mu}_{q}}_2 \leq\norm{\boldsymbol{\mu}_{p}-\boldsymbol{\mu}_{q}}_1
	    \end{align*}
	    it follows that
	    \begin{align}
	    \label{eq:inequ_impl_proof_thm1_3}
	        \norm{\boldsymbol{\mu}_{p}-\boldsymbol{\mu}_{q}}_1 \leq \frac{1}{2 C \left(m+1\right)}
		\quad\implies\quad
		D(p_i^*\Vert q_i^*) \leq C\cdot \norm{\boldsymbol{\mu}_{p_i}-\boldsymbol{\mu}_{q_i}}_2^2.
	    \end{align}
	    Therefore, if
	    \begin{align*}
	        \norm{\boldsymbol{\mu}_{p}-\boldsymbol{\mu}_{q}}_1 \leq \frac{1}{2 C \left(m+1\right)}
	    \end{align*}
	    then Eq.~\eqref{eq:inequ_impl_proof_thm1} can be further extended by
	    \begin{align*}
	        \norm{p-q}_{L^1} &\leq \sqrt{2\sum_{i=1}^N D(p_i^*\Vert q_i^*)} + \sqrt{8 \epsilon}
	        \leq \sqrt{2\sum_{i=1}^N C\cdot \norm{\boldsymbol{\mu}_{p_i}-\boldsymbol{\mu}_{q_i}}_2^2} + \sqrt{8 \epsilon}\\
	        &\leq \sqrt{2 C}\cdot \sum_{i=1}^N \norm{\boldsymbol{\mu}_{p_i}-\boldsymbol{\mu}_{q_i}}_2 + \sqrt{8 \epsilon}
	        \leq \sqrt{2 C}\cdot \sum_{i=1}^N \norm{\boldsymbol{\mu}_{p_i}-\boldsymbol{\mu}_{q_i}}_1 + \sqrt{8 \epsilon}\\
	        &= \sqrt{2 C}\cdot \norm{\boldsymbol{\mu}_{p}-\boldsymbol{\mu}_{q}}_1 + \sqrt{8 \epsilon}\\
	    \end{align*}\qed
	    
	    
	    
	\end{proof}
	
	\subsection{Proofs of Subsection~\ref{subsec:contrib_to_picture_of_prob_metrics} on our Contribution to the Picture of Probability Metrics}

	\begin{restatable}{thmrep}{thmrech}
	\label{thm:rachev}
	Let $P$ and $Q$ be two cumulative distribution functions on $\mathbb{R}$ with absolute moments $\int_{-\infty}^\infty \left|x^j\right| dP \leq a_j$ and $\int_{-\infty}^\infty \left|x^j\right| dQ \leq a_j$ of all orders $j\in\mathbb{N}$ bounded by positive numbers $a_j\in\mathbb{N}$ such that $a_1\leq a_2\leq\ldots$ form an increasing sequence.
	
	Suppose that the characteristic functions $\psi_P(t)$ and $\psi_Q(t)$ of $P$ and $Q$ fulfill
	\begin{align}
	\label{eq:bounded_zolotarev}
	    \sup_{\left|t\right|\leq T_0} |\psi_P(t)-\psi_Q(t)|\leq \varepsilon
	\end{align}
	for some real constants $T_0$ and $\varepsilon$. Then there exists an absolute constant $C_\text{Z}$ such that for all $n\in\mathbb{N}$ with
	\begin{align}
	    n^3 C_\text{Z}^{\frac{1}{n+1}} \varepsilon^{\frac{1}{n+1}}\leq a_n^{\frac{1}{n+1}} T_0/2
	\end{align}
	we have
	\begin{align}
	    \left| \int_{-\infty}^\infty x^n p\,dx - \int_{-\infty}^\infty x^n q\,dx\right|\leq C_\text{Z} a_{n+1} n^3 \varepsilon^{\frac{1}{n+1}}.
	\end{align}
	\end{restatable}
	\begin{proof}
	See~\cite[Theorem~10.3.6]{rachev2013methods}.
	\end{proof}
	To prove Lemma~\ref{lemma:rachev_bound}, the following Definition~\ref{def:metrics_of_zolotarev} and Lemma~\ref{lemma:zolotarev_bound} taken from~\cite{zolotarev1975two} are helpful.
	
	\begin{defi}
	    \label{def:metrics_of_zolotarev}
	     Zolotarev's $\lambda$-metric $d_\text{Z}$ between two cumulative distribution functions $P,Q$ on the real line is defined by~\cite{zolotarev1976metric}
	    \begin{align*}
	        d_\text{Z}(P,Q) =\min_{T > 0} \max \left\{ \frac{1}{2} \max_{\left| t\right|\leq T} |\psi_P(t)-\psi_Q(t)|, \frac{1}{T} \right\},
	    \end{align*}
	    where $\psi_P$ and $\psi_Q$ denote the characteristic functions of $P$ and $Q$.
	\end{defi}
	
	\begin{restatable}{lemmarep}{zolbound}
	\label{lemma:zolotarev_bound}
	    Let $P,Q$ be two cumulative distribution functions on the real line with probability density functions having support contained in an interval of length $2 K$. Then it holds that
	    \begin{align}
	        \label{eq:zolotarev_and_senatov}
	        d_\text{Z}(P,Q)\leq \sqrt{\left( 2 K + 24\sqrt{d_\text{L}(P,Q)}+1/2 \right) d_\text{L}(P,Q)}.
	    \end{align}
	\end{restatable}
	\begin{proof}
	See~\cite[Corollary~I]{zolotarev1975two}.
	\end{proof}
	
	\lemmarachevbound*
	
	\begin{proof}
	    Let $p,q\in\mathcal{M}([0,1])$ with respective cumulative distribution functions $P$ and $Q$ on the real line.
	    The support $[0,1]$ of $p$ and $q$ implies that $\int_{-\infty}^\infty \left|x^j\right| p\, dx\leq a_j$ and $\int_{-\infty}^\infty \left|x^j\right| q\, dx\leq a_j$ for $a_j:=1$ and $j\in\mathbb{N}$.
	    
	    Let $\varepsilon:=\sqrt{102\, d_L(P,Q)}$ and $T_0$ such that
	    \begin{align*}
	        d_\text{Z}(P,Q) =\max \left\{ \frac{1}{2} \max_{\left| t\right|\leq T_0} |\psi_P(t)-\psi_Q(t)|, \frac{1}{T_0} \right\}.
	    \end{align*}
	    Then it holds that
	    \begin{align}
	        \label{eq:proof_eq_rachevbound}
	        \begin{split}
	        \sup_{\left| t\right|\leq T_0} |\psi_P(t)-\psi_Q(t)| &\leq
	        2 \max \left\{ \frac{1}{2} \max_{\left| t\right|\leq T_0} |\psi_P(t)-\psi_Q(t)|, \frac{1}{T_0} \right\}\\
	        &\leq 2 \sqrt{\left( 2 K + 24\sqrt{d_\text{L}(P,Q)}+1/2 \right) d_\text{L}(P,Q)}\\
	        &\leq \sqrt{102\, d_L(P,Q)} =\varepsilon
	        \end{split}
	    \end{align}
	    where the second inequality follows from  Lemma~\ref{lemma:zolotarev_bound} and the last inequality follows from the fact that $d_\text{L}\leq 1$.
	    
	    Theorem~\ref{thm:rachev} can be applied and it follows that there exists an absolute constant $C_\text{Z}$ such that for all $n\in\mathbb{N}$ with
    	\begin{align*}
    	    n^3 C_\text{Z}^{\frac{1}{n+1}} \varepsilon^{\frac{1}{n+1}}\leq a_n^{\frac{1}{n+1}} T_0/2
    	\end{align*}
    	we have that
    	\begin{align*}
    	    \left| \int_{-\infty}^\infty x^n p\,dx - \int_{-\infty}^\infty x^n q\,dx\right|\leq C_\text{Z} a_{n+1} n^3 \varepsilon^{\frac{1}{n+1}}.
    	\end{align*}
    	From the definition of $\varepsilon$ and Eq.~\eqref{eq:proof_eq_rachevbound}, in particular using $\frac{2}{T_0}\leq \epsilon$, we obtain for all $n\in\mathbb{N}$ with
    	\begin{align}
    	    \label{eq:proof_zolotarev_assumption}
    	    n^3 C_\text{Z}^{\frac{1}{n+1}} \left(102\, d_L(P,Q)\right)^{\frac{n+2}{2 n+2}}\leq 1
    	\end{align}
    	the inequality
    	\begin{align}
    	    \label{eq:proof_levy_bound}
    	    \left| \int_{-\infty}^\infty x^n p\,dx - \int_{-\infty}^\infty x^n q\,dx\right|\leq C_\text{Z} n^3 \left(102\, d_L(P,Q)\right)^{\frac{1}{2 n+2}}.
    	\end{align}
    	The elements $1,\phi_1,\ldots,\phi_m$ form a basis of $\mathbb{R}_m[x]$ which implies that the value of $\norm{\boldsymbol{\mu}_p - \boldsymbol{\mu}_q}_1$ can be computed as a finite weighted sum of differences of moments (as specified by the left-hand side of Eq.~\eqref{eq:proof_levy_bound}).
    	As a consequence, the value of $\norm{\boldsymbol{\mu}_p - \boldsymbol{\mu}_q}_1$ can be upper bounded by aggregations of the right-hand side of Eq.~\eqref{eq:proof_levy_bound}.
    	Let us define $M_L$ small enough such that Eq.~\eqref{eq:proof_zolotarev_assumption} is fulfilled for all $n\leq m$.
    	From $d_L(P,Q)^{\frac{1}{2 n +2}}\leq d_L(P,Q)^{\frac{1}{2 m +2}}$ for $1\leq n\leq m$ the existence of some $C_L$ as required by Lemma~\ref{lemma:rachev_bound} follows.\qed
	\end{proof}
	
	\subsection{Proofs of Section~\ref{sec:main_result_on_learning_bounds} on our Main Result on Learning Bounds}
	\label{subsec:problem_solution_proof}

	In the following, we consider the sample case.
	\begin{restatable}{lemmarep}{sampleconvinH}
		\label{lemma:sample_convergence_in_H}
		Consider some $m\geq r\geq 2$ and some $\phim=(\phi_1,\ldots,\phi_m)^\text{T}$ such that $1,\phi_1,\ldots,\phi_m$ is an orthonormal basis of $\mathbb{R}_m[x]$.
		
		Let $p,q\in\mathcal{M}([0,1])$ such that $\log p,\log q\in W_2^r$ with Sobolev space $W_2^r$ and denote by $\widehat{\boldsymbol{\mu}}_p=\frac{1}{k}\sum_{\x\in X_p}\phim(\x)$ and $\widehat{\boldsymbol{\mu}}_q=\frac{1}{k}\sum_{\x\in X_q}\phim(\x)$ the moments of two $k$-sized samples $X_p$ and $X_q$ drawn from $p$ and $q$, respectively.
		
		If $4 e^{4\gamma + 1} e^{c_\infty/2} (m+1) \xi \leq 1$ then for all $\delta\in (0,1)$ such that
		\begin{align}
		\label{eq:assumption_on_sample_size}
		    4 C^2 (m+1)^2 m e^{-c_\infty} \leq \delta k
		\end{align}
		with probability at least $1-\delta$, the maximum entropy densities $\widehat{p}$ and $\widehat{q}$ constrained at the moments $\widehat{\boldsymbol{\mu}}_p$ and $\widehat{\boldsymbol{\mu}}_q$, respectively, exist and the following holds:	
		\begin{gather}
		\label{eq:sample_size_p}
		D(p^*\Vert \widehat{p})\leq C e^{-c_\infty} \frac{m}{k\delta}\\
		\label{eq:sample_size_q}
		D(q^*\Vert \widehat{q})\leq C e^{-c_\infty} \frac{m}{k\delta}\\
		\label{eq:sample_moment_diff}
		\norm{\widehat{\boldsymbol{\mu}}_p-\widehat{\boldsymbol{\mu}}_q}_2 \leq \frac{1}{2 (m+1) e C}\quad \implies\quad D(\widehat{p}\Vert\widehat{q})\leq e C \norm{\widehat{\boldsymbol{\mu}}_p-\widehat{\boldsymbol{\mu}}_q}_2^2
		\end{gather}
		where
		\begin{align}
		    c_\infty &:=\max\left\{\norm{\log p}_\infty,\norm{\log q}_\infty\right\}\\
		    c_r &:=\max\{\norm{\partial_x^r \log p}_{L^2}\norm{\partial_x^r \log q}_{L^2}\}
		\end{align}
		and $\gamma,\xi$ and $C$ are defined as in Lemma~\ref{lemma:convergence_in_H}.
	\end{restatable}
	
	\begin{proof}	
	    Let $m,r$ be such that $m\geq r\geq 2$.
		Let $\phim=(\phi_1,\ldots,\phi_m)^\text{T}$ such that $1,\phi_1,\ldots,\phi_m$ is an orthonormal basis of $\mathbb{R}_m[x]$.
		
		Let $p,q\in\mathcal{M}([0,1])$ such that $\log p,\log q\in W_2^r$ with Sobolev space $W_2^r$.
		Let further $\widehat{\boldsymbol{\mu}}_p=\frac{1}{k}\sum_{\x\in X_p}\phim(\x)$ and $\widehat{\boldsymbol{\mu}}_p=\frac{1}{k}\sum_{\x\in X_q}\phim(\x)$ be the moments of two $k$-sized samples $X_p$ and $X_q$ drawn from $p$ and $q$, respectively.
		
		From Lemma~\ref{lemma:Ap_bound} we obtain ${A_p:=e^{\norm{\log p}_\infty/2} (m+1)}$ and ${A_q:=e^{\norm{\log q}_\infty/2} (m+1)}$ such that ${\norm{f_m}_\infty \leq A_p \norm{f_m}_{L^2(p)}}$ and ${\norm{f_m}_\infty \leq A_q \norm{f_m}_{L^2(q)}}$ for all $f_m\in\mathbb{R}_m[x]$.
		
		Denote by $\tilde A:=\max\{A_p,A_q\}$ and by $f_p:=\log p, f_q:=\log q$.
		Further denote by
		\begin{align*}
		    \tilde \gamma:=\max \left\{\min_{f_m\in\mathbb{R}_m[x]} \norm{f_p-f_m}_\infty, \min_{f_m\in\mathbb{R}_m[x]} \norm{f_q-f_m}_\infty\right\}
		\end{align*}
		and
		\begin{align*}
		    \tilde \xi:=\max\left\{ \min_{f_m\in\mathbb{R}_m[x]} \norm{f_p-f_m}_{L^2(p)}, \min_{f_m\in\mathbb{R}_m[x]} \norm{f_q-f_m}_{L^2(p)} \right\}
		\end{align*}
		minimal errors of approximating $f_p$ and $f_q$ by polynomials $f_m\in \mathbb{R}_m[x]$.
		Denote by ${\tilde b:=e^{2 \tilde\gamma + 4 e^{4 \tilde\gamma + 1} \tilde\xi \tilde A}}$.
		
		If $4 e^{4\tilde\gamma+1} \tilde A \tilde\xi\leq 1$, then Corollary~\ref{lemma:barron_proof_thm3} implies that
		\begin{align}
		\label{eq:barron_log_bound}
		\norm{\log{p/p^*}}_\infty \leq 2 \tilde\gamma + 4 e^{4 \tilde\gamma + 1} \tilde\xi \tilde A
		\end{align}
		and for all $\delta\in (0,1)$ such that $(4 e \tilde b \tilde A)^2 m \leq \delta k$. Corollary~\ref{cor:barron_sample} implies the existence of the maximum entropy densities $\widehat{p}$ and $\widehat q$ with probability at least $1-\delta$ and it holds that
		\begin{align}
		\label{eq:barron_sample_appl1}
		    D(p^*\Vert \widehat{p})\leq 2 e \tilde b \frac{m}{k\delta}\\
		    \label{eq:barron_sample_appl2}
		    D(q^*\Vert \widehat{q})\leq 2 e \tilde b \frac{m}{k\delta}\\
		    \label{eq:barron_sample_appl3}
		    \norm{\log p^*/\widehat{p}}_\infty\leq 1.
		\end{align}
		Consider $\gamma$ and $\xi$ as defined in Eq.~\eqref{eq:gamma} and Eq.~\eqref{eq:xi}, respectively.
		Note that
		\begin{align*}
		    \min_{f_m\in\mathbb{R}_m[x]} \norm{f_p-f_m}_{L^2(p)}^2 &= \min_{f_m\in\mathbb{R}_m[x]} \int |f_p-f_m|^2 p\\
		    &\leq \sup \left|p\right| \min_{f_m\in\mathbb{R}_m[x]} \int |f_p-f_m|^2\\
		    &\leq e^{c_\infty} \min_{f_m\in\mathbb{R}_m[x]} \norm{f_p-f_m}_{2}^2.
		\end{align*}
		Lemma~\ref{lemma:cox} yields $\tilde \gamma\leq \gamma, \tilde \xi\leq \xi$.
		It also holds that $\tilde A=\max\{A_p,A_q\}\leq e^{c_\infty/2} (m+1)$ which implies
		\begin{align*}
		    (4 e \tilde b \tilde A)^2 m &\leq (4 e e^{2 \gamma + 4 e^{4 \gamma + 1} \xi \tilde A} \tilde A)^2 m\\
		    &\leq (4 e e^{2 \gamma + 4 e^{4 \gamma + 1} \xi e^{c_\infty/2} (m+1)} e^{c_\infty/2} (m+1))^2 m\\
		    &= 4 C^2 (m+1)^2 m e^{-c_\infty}.
		\end{align*}
		Therefore, if $4 e^{4\gamma + 1} e^{c_\infty/2} (m+1) \xi \leq 1$ then for all $\delta\in (0,1)$ such that
		$$
		4 C^2 (m+1)^2 m e^{-c_\infty} \leq \delta k
		$$
		with probability at least $1-\delta$ the maximum entropy densities $\widehat{p}$ and $\widehat{q}$ constrained at the moments $\widehat{\boldsymbol{\mu}}_p$ and $\widehat{\boldsymbol{\mu}}_q$, respectively, exist, and the following inequalities hold:
		\begin{align}
		    D(p^*\Vert \widehat{p}) &\leq 2 e \tilde b \frac{m}{k\delta}\leq 2 e e^{2 \tilde\gamma + 4 e^{4 \tilde\gamma + 1} \tilde\xi \tilde A} \frac{m}{k\delta}\leq C e^{-c_\infty} \frac{m}{k\delta}\\
		    D(q^*\Vert \widehat{q}) &\leq C e^{-c_\infty} \frac{m}{k\delta}\\
		    \label{eq:proof_sample1}
		    \norm{\log p^*/\widehat{p}}_\infty &\leq 1\\
		    \label{eq:proof_sample2}
		    \norm{\log{p/p^*}}_\infty &\leq 2 \tilde\gamma + 4 e^{4 \tilde\gamma + 1} \tilde\xi \tilde A
		    \leq 2 \gamma + 4 e^{4 \gamma + 1} \xi e^{c_\infty/2} (m+1)
		\end{align}
		where the last inequality follows from Eq.~\eqref{eq:barron_log_bound}.
		
		Let us now prove the upper bound on $D(\widehat{p}\Vert\widehat{q})$.
        To do this, note that $1,\phi_1,\ldots,\phi_m$ form an orthonormal basis of $\mathbb{R}_m[x]$, i.e. they form an orthonormal basis of $\mathbb{R}_m[x]$ \wrt~the uniform weight function $\tilde q$ on $[0,1]$.
		For $\tilde q$ it holds that $\norm{\log\tilde q}_\infty < \infty$ and with $A_{\tilde q}=m+1$, due to Lemma~\ref{lemma:barron_A_bound}, it also holds that
		\begin{align*}
		\norm{f_m}_\infty \leq A_{\tilde q} \norm{f_m}_{L^2(\tilde q)}
		\end{align*}
		for all  $f_m\in\mathbb{R}_m[x]$.
		Consider the vector of moments $\boldsymbol{\tilde\mu}:=\widehat{\boldsymbol{\mu}}_q\in [0,1]^m$.
		Let $\tilde p_0:=\widehat p\in\mathcal{M}([0,1])$ and note that its moments are given by $\int \phim \tilde p_0=\widehat{\boldsymbol{\mu}}_p$.
		Let $\tilde b :=e^{\norm{\log \tilde q/\widehat{p}}_\infty}$.
		If the maximum entropy densities $\widehat{p}$ and $\widehat{q}$ constrained at the moments $\widehat{\boldsymbol{\mu}}_p$ and $\widehat{\boldsymbol{\mu}}_q$ exist, then by Lemma~\ref{lemma:barron_and_sheu_lemma5} it holds that
		\begin{align*}
		    \norm{\widehat{\boldsymbol{\mu}}_p-\widehat{\boldsymbol{\mu}}_q}_2\leq\frac{1}{4 A_{\tilde q} e \tilde b}\quad\implies\quad D(\widehat p\Vert\widehat q)\leq 2 e^t \tilde b \norm{\widehat{\boldsymbol{\mu}}_p-\widehat{\boldsymbol{\mu}}_q}_2^2
		\end{align*}
		especially for $t$ such that $4 e \tilde b A_{\tilde q} \norm{\widehat{\boldsymbol{\mu}}_p-\widehat{\boldsymbol{\mu}}_q}_2\leq t\leq 1$.
		If Eq.~\eqref{eq:proof_sample1} and Eq.~\eqref{eq:proof_sample2} hold, then
		\begin{align*}
		    \tilde b &\leq e^{\norm{\log \tilde (\tilde q p p^*) / (p p^* \widehat{p})}_\infty}\\
		    &\leq e^{\norm{\log p}_\infty +\norm{\log p/p^*}_\infty + \norm{\log p^*/\widehat{p}}_\infty}\\
		    &\leq e^{c_\infty + 2 \gamma + 4 e^{4 \gamma + 1} \xi e^{c_\infty/2} (m+1) + 1}\\
		    &= \frac{1}{2} C
		\end{align*}
		for $C$ as defined in Lemma~\ref{lemma:convergence_in_H}.
		Therefore, if $4 e^{4\gamma + 1} e^{c_\infty/2} (m+1) \xi \leq 1$ then for all $\delta\in [1,0)$ such that
		$$
		4 C^2 (m+1)^2 m e^{-c_\infty} \leq \delta k
		$$
		with probability at least $1-\delta$ the maximum entropy densities $\widehat{p}$ and $\widehat{q}$ constrained at the moments $\widehat{\boldsymbol{\mu}}_p$ and $\widehat{\boldsymbol{\mu}}_q$, respectively, exist, and, since Eq.~\eqref{eq:proof_sample1} and Eq.~\eqref{eq:proof_sample2} hold, the following also holds:
		\begin{align*}
		    \norm{\widehat{\boldsymbol{\mu}}_p-\widehat{\boldsymbol{\mu}}_q}_2\leq\frac{1}{2 (m+1) e C}\quad\implies\quad D(\widehat p\Vert\widehat q)\leq e C \norm{\widehat{\boldsymbol{\mu}}_p-\widehat{\boldsymbol{\mu}}_q}_2^2.
		\end{align*}\qed
	\end{proof}
	
	\begin{remark}
		\label{remark:simplification_min_sample_size}
		If the densities $p,q\in\mathcal{H}_{m,\epsilon}$ then Lemma~\ref{lemma:simplification} allows to replace the assumption in Eq.~\eqref{eq:assumption_on_sample_size} of Lemma~\ref{lemma:sample_convergence_in_H} by the assumption
		\begin{align}
		\frac{1}{\delta} 16 e^{3m-1} (m+1)^2 m\leq k.
		\end{align}
		However, smaller lower bounds on the sample size are obtained by using the definition of $C$ as in Lemma~\ref{lemma:convergence_in_H}.
	\end{remark}
	We are now able to prove our main result.
	
	\mainresult*
	
	\begin{proof}
	    Consider some $m, \epsilon,\phim$ and $\mathcal{H}_{m,\epsilon}$ as in Definition~\ref{def:H} and a function class $\mathcal{F}$ with finite VC-dimension.
		Let $p,q\in\mathcal{H}_{m,\epsilon}$ and ${l_p,l_q:[0,1]^N\to [0,1]}$.
		Let $X_p$ and $X_q$ be two arbitrary $k$-sized samples drawn from $p$ and $q$, respectively.
		
		Eq.~\eqref{eq:simple_da_equation} (proven by Ben-David et al.~\cite{ben2010theory}) implies that
		\begin{align}
    	\E_{q}\big[|f-l_q|\big]\leq \E_{p}\big[|f-l_p|\big] + \norm{p-q}_{L^1} + \lambda^*
    	\end{align}
    	where $\lambda^* = \inf_{h\in\mathcal{F}}\big(\E_p[|h-l_p|]+\E_q[|h-l_q|]\big)$.
    	Combining Eq.~\eqref{eq:simple_da_equation} with Eq.~\eqref{eq:vapnik_bound} (proven by Vapnik and Chervonenkis~\cite{vapnik2015uniform}) the following holds with probability at least $1-\delta$ (over the choice of $k$-sized samples $X_q$ drawn from $q$):
    	\begin{align}
        \begin{split}
            \E_{q}\big[|f-l_q|\big] &\leq \frac{1}{k}\sum_{\x\in X_p}|f(\x)-l(\x)|
            + \sqrt{\frac{4}{k} \left( d\log \frac{2 e k}{d} + \log\frac{4}{\delta} \right)} + \lambda^*\\
            &\phantom{\leq} + \norm{p-q}_{L^1}
        \end{split}
    	\end{align}
	    
	    In the following, we bound the term $\norm{p-q}_{L^1}$ from above to obtain the second line of Eq.~\eqref{eq:moment_adapt_result_bound}:
	    If the maximum entropy densities $\widehat{p}$ and $\widehat{q}$ constrained at the moments $\widehat{\boldsymbol{\mu}}_p=\frac{1}{k}\sum_{\x\in X_p}\boldsymbol{\phi}_m(\x)$ and $\widehat{\boldsymbol{\mu}}_q=\frac{1}{k}\sum_{\x\in X_q}\boldsymbol{\phi}_m(\x)$ exist, then the Triangle inequality and Pinsker's inequality imply
	    \begin{align*}
	        \norm{p-q}_{L^1} &\leq \norm{\widehat{p}-\widehat{q}}_{L^1} +\norm{\widehat{p}-p^*}_{L^1} + \norm{\widehat{q}-q^*}_{L^1} + \norm{p^*-p}_{L^1} + \norm{q^*-q}_{L^1}\\
	        &\leq \norm{\widehat{p}-\widehat{q}}_{L^1} +\norm{\widehat{p}-p^*}_{L^1} + \norm{\widehat{q}-q^*}_{L^1} + \sqrt{2 D(p\Vert p^*)} + \sqrt{2 D(q\Vert q^*)}
	    \end{align*}
	    which, by the $\epsilon$-closeness of $p,q\in\mathcal{H}_{m,\epsilon}$, further implies that
	    \begin{align}
	    \label{eq:sample_proof_l1_inequ}
            \norm{p-q}_{L^1} 
	        &\leq \norm{\widehat{p}-\widehat{q}}_{L^1} +\norm{\widehat{p}-p^*}_{L^1} + \norm{\widehat{q}-q^*}_{L^1} + \sqrt{8\epsilon}\\
	        \nonumber
	        &\leq \sqrt{D(\widehat{p}\Vert \widehat{q})} +\sqrt{D(\widehat{p}\Vert p^*)} + \sqrt{D(\widehat{q}\Vert q^*)} + \sqrt{8\epsilon}.
	    \end{align}
	    The vector $\phim=(\phi_1,\ldots,\phi_{m N})^\text{T}$ is a polynomial vector such that $1,\phi_1,\ldots,\phi_{m N}$ is an orthonormal basis of ${\mathrm{Span}()}$. 
	    Therefore, by applying Lemma~\ref{lemma:independence}, we obtain
	    \begin{align}
	        \label{eq:inequ_impl_proof_main_sample_thm}
	        \norm{p-q}_{L^1} \leq \sqrt{2\sum_{i=1}^N D(\widehat{p}_i\Vert \widehat{q}_i)} + \sqrt{2\sum_{i=1}^N D(\widehat{p}_i\Vert p^*_i)} + \sqrt{2\sum_{i=1}^N D(\widehat{q}_i\Vert q^*_i)} + \sqrt{8 \epsilon}
	    \end{align}
	    where $p_i^*$ and $\widehat{p}_i$ denote the maximum entropy densities of $p$ and $\widehat{p}$ constrained at the moments $\int \boldsymbol{\phi}_m^{(i)} p$ and $\int \boldsymbol{\phi}_m^{(i)} \widehat{p}$, respectively, for some vector $\boldsymbol{\phi}_m^{(i)}=(\phi_{i1},\ldots,\phi_{im})$ such that $1,\phi_{i1},\ldots,\phi_{im}$ is an orthonormal basis of $\mathbb{R}_m[x_i]$.
	    
	    The density $p_i^*$ is the maximum entropy density constrained at the moments $\boldsymbol{\mu}_{p_i}:=\int_0^1 \boldsymbol{\phi}_m^{(i)} p_i$ for the marginal density $p_i$ of $p$ defined by
	    \begin{align*}
	        p_i(x_i):=\int_0^1\cdots\int_0^1 p(x_1,\ldots,x_N)\, d x_1\cdots d x_{i-1} d x_{i+1}\cdots d x_N.
	    \end{align*}
	    Denote by $X_{p_i}$ the $k$-sized sample (multiset) consisting of the $i$-th coordinates of the vectors stored in the sample $X$.
	    It holds that the sample $X_{p_i}$ is drawn from the probability density $p_i$ and the density $\widehat{p}_i$ can be seen to be the maximum entropy density constrained at the moments $\widehat{\boldsymbol{\mu}}_{p_i}=\frac{1}{k}\sum_{\x\in X_{p_i}}\boldsymbol{\phi}_m^{(i)}(\x)$.
	    From Definition~\ref{def:H} it follows that $\norm{\partial_{x_i}^m \log p_i}_{L^2}\leq 5^{m-4}$ and therefore ${\log p_i\in W_2^r}$ with Sobolev space $W_2^r$.
	    All assumptions from Lemma~\ref{lemma:sample_convergence_in_H} are fulfilled and therefore the following holds:
	    If $4 e^{4\gamma + 1} e^{c_\infty/2} (m+1) \xi \leq 1$ then for all $\delta\in (0,1)$ such that
		\begin{align*}
		    4 C^2 (m+1)^2 m e^{-c_\infty} \leq \delta k
		\end{align*}
		with probability at least $1-\delta$ the maximum entropy densities $\widehat{p}_i$ and $\widehat{q}_i$ constrained at the moments $\widehat{\boldsymbol{\mu}}_{p_i}$ and $\widehat{\boldsymbol{\mu}}_{q_i}$, respectively, exist and the following holds:	
		\begin{gather}
		\label{eq:proof_sample_size_p}
		D(p_i^*\Vert \widehat{p}_i)\leq C e^{-c_\infty} \frac{m}{k\delta}\\
		\label{eq:proof_sample_size_q}
		D(q_i^*\Vert \widehat{q}_i)\leq C e^{-c_\infty} \frac{m}{k\delta}\\
		\label{eq:proof_sample_moment_diff}
		\norm{\widehat{\boldsymbol{\mu}}_{p_i}-\widehat{\boldsymbol{\mu}}_{q_i}}_2 \leq \frac{1}{2 (m+1) e C}\quad \implies\quad D(\widehat{p}_i\Vert\widehat{q}_i)\leq e C \norm{\widehat{\boldsymbol{\mu}}_{p_i}-\widehat{\boldsymbol{\mu}}_{q_i}}_2^2
		\end{gather}
		with
		\begin{align*}
		    c_\infty &:=\max\left\{\norm{\log p_i}_\infty,\norm{\log q_i}_\infty\right\}\\
		    c_r &:=\max\{\norm{\partial_x^r \log p_i}_{L^2}\norm{\partial_x^r \log q_i}_{L^2}\}
		\end{align*}
		and $\gamma,\xi$ and $C$ are defined as in Lemma~\ref{lemma:convergence_in_H}.
		Since $p,q\in\mathcal{H}_{m,\epsilon}$, Lemma~\ref{lemma:simplification} implies that
		\begin{align*}
    		4 e^{4\ga + 1} e^{c_\infty/2} (m+1) \de\leq 1\quad\text{and}\quad C\leq 2 e^{(3m-1)/2}
		\end{align*}
		and by Remark~\ref{remark:simplification_min_sample_size} we may simplify the assumption in Eq.~\eqref{eq:assumption_on_sample_size} and obtain
		\begin{align*}
		    4 C^2(m+1)^2 m \delta^{-1} \leq k
		\end{align*}
		as alternative.
		
		Combining the bounds in Eq.~\eqref{eq:proof_sample_size_p}, Eq.~\eqref{eq:proof_sample_size_q} and Eq.~\eqref{eq:proof_sample_moment_diff} with the bound on the $L^1$-difference in Eq.~\eqref{eq:inequ_impl_proof_main_sample_thm}, yields the following statement.
		For every $\delta\in (0,1)$ and all $f\in\mathcal{F}$ the following holds with probability at least $1-\delta$ (over the choice of samples):
		If
		\begin{flalign*}
		4 C^2(m+1)^2 m \delta^{-1} \leq k\\
		\norm{\widehat{\boldsymbol{\mu}}_p-\widehat{\boldsymbol{\mu}}_q}_1 \leq \left(2 (m+1) e C\right)^{-1}
		\end{flalign*}
		then the maximum entropy densities $\widehat{p}$ and $\widehat{q}$ exist and it holds that
		\begin{align}
		    \norm{p-q}_{L^1} &\leq \sqrt{2\sum_{i=1}^N D(\widehat{p}_i\Vert \widehat{q}_i)} + \sqrt{2\sum_{i=1}^N D(\widehat{p}_i\Vert p^*_i)} + \sqrt{2\sum_{i=1}^N D(\widehat{q}_i\Vert q^*_i)} + \sqrt{8 \epsilon}\nonumber\\
		    &\leq \sqrt{2\sum_{i=1}^N e C \norm{\widehat{\boldsymbol{\mu}}_{p_i}-\widehat{\boldsymbol{\mu}}_{q_i}}_2^2} + 2 \sqrt{2\sum_{i=1}^N C e^{-c_\infty} \frac{m}{k\delta}} + \sqrt{8 \epsilon}\nonumber\\
		    &\leq \sqrt{2 e C} \sum_{i=1}^N { \norm{\widehat{\boldsymbol{\mu}}_{p_i}-\widehat{\boldsymbol{\mu}}_{q_i}}_2} + \sqrt{8 C \frac{N m}{\delta k}} e^{-c_\infty/2} + \sqrt{8 \epsilon}\nonumber\\
		    \label{eq:sample_bound_improved}
		    &\leq \sqrt{2 e C} \norm{\widehat{\boldsymbol{\mu}}_p-\widehat{\boldsymbol{\mu}}_q}_1 + \sqrt{8 C \frac{N m}{k\delta}} + \sqrt{8\epsilon}
		\end{align}
		where the last inequality is due to the fact that $e^{-c_\infty/2}\leq 1$ and the inequality $\norm{\x}_2\leq\norm{\x}_1$.\qed
	\end{proof}

	\section{Conclusion and Future Work}
	\label{sec:conclusion}




In this work we studied domain adaptation under weak assumptions on the similarity of source and target distributions.
Our assumptions are based on moment distances which realize weaker similarity concepts than most other common probability metrics.
We formalize the novel problem setting, give conditions for the convergence of a discriminative model in this setting and derive bounds describing its generalization ability.
For smooth densities with weakly coupled marginals, our conditions can be made as precise as required based on the number of moments and the smoothness of the distributions.
Our focus on studying weak assumptions on the similarity of distributions enables straightforward extensions using stronger assumptions, e.g.~new learning bounds based on relations between probability metrics.

We would primarily like to extend the proposed bounds on the difference between distributions by further upper bounding the entropy-based  term  in  terms  of  smoothness of log-densities as it is done e.~g.~in~\cite{barron1991approximation} for one dimension.
Such bounds can lead to estimates of the number of moments needed such that an underlying smooth distribution is defined up to prescribed accuracy, which is, to the best of our knowledge, an open problem~\cite{schmudgen2017moment,tardella2001note}.
Concerning improved algorithms for domain adaptation, future plans are centered around entropy minimization as suggested by our generalization bound.
Generally in industrial applications with low sample sizes we consider, due to its robustness and weak assumptions, significant potential for moment distance based domain adaptation as a starting point for developing more problem-specific distance concepts.

\section*{Acknowledgements}

We thank Sepp Hochreiter, Helmut Gfrerer, Thomas Natschl\"ager and Hamid Eghbal-Zadeh for helpful discussions. The research reported in this paper has been funded by the Federal Ministry for Climate Action, Environment, Energy, Mobility, Innovation and Technology (BMK), the Federal Ministry for Digital and Economic Affairs (BMDW), and the Province of Upper Austria in the frame of the COMET--Competence Centers for Excellent Technologies Programme and the COMET Module S3AI managed by the Austrian Research Promotion Agency FFG. The first and second author further acknowledge the support of the FFG in the project AutoQual-I.
	
	\bibliographystyle{plain}
	\bibliography{mom_dist}
	
\end{document}